\documentclass{article}

\usepackage[english]{babel}
\usepackage[utf8]{inputenc}
\usepackage[T1]{fontenc}
\usepackage[a4paper,left=3cm,right=3cm,top=3cm,bottom=3.2cm]{geometry} 
\usepackage{tikz}
\usetikzlibrary[topaths] 

\usepackage{amsmath, mathrsfs, amssymb, amsthm, mathtools, bbm, bm} 

\usepackage{subfig} 

\usepackage{caption,float}

\usepackage{makecell} 
\usepackage{enumitem} 

\newtheorem{theorem}{Theorem}

\newtheorem{lemma}{Lemma}
\newtheorem{remark}{Remark}
\newtheorem{corollary}{Corollary}
\newtheorem{proposition}{Proposition}

\newtheorem*{theorem*}{Theorem}
\newtheorem*{example*}{Example} 
\newtheorem*{definition*}{Definition}
\newtheorem*{lemma*}{Lemma}
\newtheorem*{remark*}{Remark}
\newtheorem*{corollary*}{Corollary}
\newtheorem*{proposition*}{Proposition}
\newtheorem*{assumption*}{Assumption}
\newtheorem*{claim*}{Claim}

\newtheoremstyle{TheoremNum}
        {\topsep}{\topsep}              
        {\itshape}                      
        {}                              
        {\bfseries}                     
        {.}                             
        { }                             
        {\thmname{#1}\thmnote{ \bfseries #3}}
\theoremstyle{TheoremNum}

\newtheoremstyle{LemmaNum}
        {\topsep}{\topsep}              
        {\itshape}                      
        {}                              
        {\bfseries}                     
        {.}                             
        { }                             
        {\thmname{#1}\thmnote{ \bfseries #3}}
\theoremstyle{LemmaNum}

\usepackage{hyperref} 

\usepackage{natbib, float}

\usepackage{amsmath, mathrsfs, amssymb, amsthm, mathtools, bbm, bm} 

\usepackage{url}

\usepackage{xcolor} 

\usepackage{float}

\usepackage{makecell} 

\usepackage{subfig}

\usepackage{algorithm} 
\usepackage{algcompatible}

\newcommand{\adj}{ \mathrm{adj} } 
\newcommand{\CHA}{ \mathrm{\texttt{CHA}}}

\newcommand{\St}{\mathrm{St}}

\newcommand{\e}{ \bm{ e } } 

\newcommand{\Stein}{\mathrm{Stein}}

\newcommand{\entry}{\mathrm{entry}}

\renewcommand{\H}{ \mathrm{H} }
\newcommand{\A}{ \mathrm{A} }

\newcommand{\I}{ \mathcal{I} }

\renewcommand{\Pr}{ \mathbb{P} }

\newcommand{\ol}{ \overline }

\newcommand{\Hess}{\mathrm{Hess}}

\newcommand{\inj}{ \mathrm{inj} }

\newcommand{\M}{\mathcal{M}}

\newcommand{\Exp}{{\mathrm{Exp}}}

\newcommand{\R}{\mathbb{R}}
\newcommand{\B}{\mathbb{B}}
\renewcommand{\S}{\mathbb{S}}
\newcommand{\E}{\mathbb{E}}

\renewcommand{\[}{\left[ }
\renewcommand{\]}{\right] }

\newcommand{\<}{\left< }
\renewcommand{\>}{\right> }

\renewcommand{\(}{\left( }
\renewcommand{\)}{\right) }

\newcommand{\wt}{\widetilde }
\newcommand{\wh}{\widehat }

\begin{document} 

\title{On Sharp Stochastic Zeroth Order Hessian Estimators over Riemannian Manifolds} 

\author{Tianyu Wang\footnote{wangtianyu@fudan.edu.cn}}  

\date{} 

\maketitle

\begin{abstract}
    We study Hessian estimators for functions defined over an $n$-dimensional complete analytic Riemannian manifold. We introduce new stochastic zeroth-order Hessian estimators using $O (1)$ function evaluations. We show that, for an analytic real-valued function $f$, our estimator achieves a bias bound of order $ O \left( \gamma \delta^2 \right) $, where $ \gamma $ depends on both the Levi-Civita connection and function $f$, and $\delta$ is the finite difference step size. To the best of our knowledge, our results provide the first bias bound for Hessian estimators that explicitly depends on the geometry of the underlying Riemannian manifold. We also study downstream computations based on our Hessian estimators. The supremacy of our method is evidenced by empirical evaluations.
\end{abstract}

\section{Introduction} 

Hessian computation is one of the central tasks in optimization, machine learning and related fields. Understanding the landscape of the objective function is in many cases the first step towards solving a mathematical programming problem, and Hessian is one of the key quantities that depict the function landscape.  
Often in real-world scenarios, the objective function is a black-box, and its Hessian is not directly computable. In these cases, zeroth-order Hessian computation techniques are needed if one wants to understand the function landscape via its Hessian. 

To this end, we introduce new zeroth-order methods for estimating a function's Hessian at any given point over an $n$-dimensional complete analytic Riemannian manifold $(\M, g)$. 
For $p \in \M$ and an analytic real-valued function $f$ defined over a complete analytic Riemannian manifold $\M$, the Hessian estimator of $f$ at $ p $ is 
\begin{align} 
    \widehat{\H} {f} (p; v,w; \delta) := \frac{n^2}{\delta^2} f ( \Exp_p ( \delta v + \delta w) ) v \otimes w , \label{eq:def-est} 
\end{align} 
where $\Exp_p$ is the exponential map, $v, w$ are independently sampled from the unit sphere in $T_p \M$, $v\otimes w$ denotes the tensor product of $v$ and $w$ ($v,w \in T_p \M$), and $\delta$ is the finite-difference step size. 

Our Hessian estimator satisfies 
\begin{align*} 
    &\; \left\| \E_{ v,w \overset{i.i.d.}{\sim} \S_p } \[ \wh{\H} f (p; v,w; \delta ) \] - \Hess f (p) \right\| \\
    =& \;  O \( \delta^2 \sup_{u \in \S_p} \left| \E_{ v,w \overset{i.i.d.}{\sim} \S_p } \[ \frac{n}{n+2} \nabla_{ u}^2 \( \nabla_v^2 + \nabla_w^2 \) f (p)  \] \right| \), 
\end{align*} 
where $\| \cdot \|$ is the $\infty$-Schatten norm, $ \S_p$ is the unit sphere in $T_p \M$, $\Hess f (p)$ is the Hessian of $f$ at $p$, and $\nabla $ is the covariant derivative associated with the Riemannian metric. 

This bias bound improves previously known results in two ways: 
\begin{enumerate} 
    \item It provides, via the Levi-Civita connection, the first bias bound for Hessian estimators that explicitly depends on the local geometry of the underlying space; 
    \item It significantly improves best previously known bias bound for $O(1)$-evaluation Hessian estimators over Riemannian manifolds, which is of order $ O(L_2 n^2 \delta) $, where $ L_2 $ is the Lipschitz constant for the Hessian, $n$ is the dimension of the manifold, and $\delta$ is the finite-difference step size. See Remark \ref{remark} for details. 
\end{enumerate} 

We also study downstream computations for our Hessian estimator. More specifically, we introduce novel provably accurate methods for computing adjugate and inversion of the Hessian matrix, all using zeroth order information only. These zeroth order computation methods may be used as primers for further applications. The supremacy of our method over existing methods is evidenced by careful empirical evaluations. 

\subsection*{Related Works} 

Zeroth order optimization has attracted the attention of many researchers. Under this broad umbrella, there stands the Bayesian optimization methods (See the review article by \citet[][]{shahriari2015taking} for an overview), comparison-based methods \citep[e.g.,][]{nelder1965simplex}, genetic algorithms \citep[e.g.,][]{goldberg1988genetic}, best arm identification from the multi-armed bandit community \citep[e.g.,][]{bubeck2009pure,audibert2010best}, and many others (See the book by \citet[][]{conn2009introduction} for an overview). Among all these zeroth order optimization schemes, one classic and prosperous line of works focuses on estimating higher order derivatives using zeroth order information.

Zeroth order gradient estimators make up a large portion of derivative estimation literature. In the past few years, \citet{flaxman2005online} studied the stochastic gradient estimator using a single-point for the purpose of bandit learning. \citet{duchi2015optimal} studied stabilization of the stochastic gradient estimator via two-points (or multi-points) evaluations. \citet{nesterov2017random,li2020stochastic} studied gradient estimators using Gaussian smoothing, and investigated downstream optimization methods using the estimated gradient. Recently, \citet{wang2021GW} studied stochastic gradient estimators over Riemannian manifolds, via the lens of the Greene-Wu convolution.

Zeroth order Hessian estimation is also a central topic in derivative estimation. In the control community, gradient-based Hessian estimators were introduced for iterative optimization algorithms, and asymptotic convergence was proved \citep{spall2000adaptive}. 
Apart from this asymptotic result, no generic non-asymptotic bound for $ O (1) $-evaluation Hessian estimators are well investigated until recently. Based on the Stein's identity \citep{10.1214/aos/1176345632}, \citet{balasubramanian2021zeroth} designed the Stein-type Hessian estimator, and combined it with cubic regularized Newton's method \citep{nesterov2006cubic} for non-convex optimization. \citet{li2020stochastic} generalizes the Stein-type Hessian estimators to Riemannian manifolds embedded in Euclidean spaces. Several authors have also considered variance and higher order moments of Hessian (and gradient) estimators \citep{li2020stochastic,balasubramanian2021zeroth,fengwang2022}. In particular, \citet{fengwang2022} showed that estimators via random orthogonal frames from Steifel's manifolds have significantly smaller variance.
Yet in the case of non-trivial curvature \citep{li2020stochastic}, no geometry-aware bias bound has been given prior to our work. 

\section{Preliminaries and Conventions}
\label{sec:pre}


For better readability, we list here some notations and conventions that will be used throughout the rest of this paper. 

\begin{itemize}
    \item For any $ p \in \M$, let $ U_p $ denote the open set near $p$ that is diffeomorphic to a subset of $\R^n$ via the local normal coordinate chart $\phi$. Define the distance $d_p (q_1, q_2)$ ($q_1,q_2 \in U_p$) such that  
    \begin{align*} 
        d_p (q_1, q_2) = d_{\text{Euc}} ( \phi (q_1) , \phi (q_2) ). 
    \end{align*} 
    where $d_{\text{Euc}}$ is the Euclidean distance in $\R^n$. 
    
    \item \textbf{(A0, Analyticity Assumption)} Throughout the paper, we assume that, both the Riemannian metric and the function of interest are analytic.
    
    \item The injectivity radius of $p \in \M$ (written $\inj (p)$) is defined as the radius of the largest geodesic ball that is contained in $U_p$. \textbf{(A1, Small Step Size Assumption)} Throughout the paper, 
    we assume that the finite difference step size $\delta$ of the estimator at point $p \in \M$ satisfies $\delta \le \frac{\inj (p)}{2}$. 

    \item All musical isomorphisms are omitted when there is no confusion. 
    
    
    \item For any $p \in \M$ and $\alpha > 0$, we use $ \alpha \S_p $ (resp. $\alpha \B_p$) to denote the origin-centered sphere (resp. ball) in $T_p \M$ with radius $\alpha$. For simplicity, we  write $ \S_p = 1 \S_p $ (resp. $\B_p = 1 \B_p$). It is worth emphasizing that $\S_p$ and $ \B_p $ are in $T_p \M$. They are different from geodesic balls which reside in $\M$. 
    
    
    \item For $p \in \M$ and $q \in U_p $, we use $ \I_p^q: T_p \M \rightarrow T_{q} \M $ to denote the parallel transport from $ T_p \M $ to $ T_q \M $ along the distance-minimizing geodesic connecting $p$ and $q$. For any $p \in \M$, $ u \in T_p \M $ and $q \in U_p$, define $u_{q} = \I_{p}^{q} (u)$. 
    More generally, $ \I_p^q $ denotes the parallel transport along the distance-minimizing geodesic from $p$ to $q$, among the fiber bundle that is compatible with the Riemannian structure. 
    
    \item We will use the double exponential map notation \citep{gavrilov2007double}. For any $ p \in \M$ and $ u , v \in T_p \M $ such that $ \Exp_p (u) \in U_p $, we write $ \Exp_p^2 ( u, v ) = \Exp_{\Exp_{p} (u)} (v_{\Exp_{p} (u) }) $. 
    
    \item \textbf{(Definition of Hessian \citep[e.g.,][]{petersen2006riemannian})} Over an $n$-dimensional complete Riemannian manifold $\M$, the Hessian of a smooth function $f : \M \rightarrow \R $ at $p$ is a bilinear form $ \Hess f (p) : T_p \M \times T_p \M \rightarrow \R$ such that, for all $u,v \in T_p \M$, $ \Hess f (p) (u,v) = \< \nabla_v d f \big|_p , u \> $. Since the Levi-Civita connection is torsion-free, the Hessian is symmetric: $\Hess f (p) (u,v) = \Hess f (p) (v,u)$ for all $u,v \in T_p \M$. For a smooth function $f$, its Hessian satisfies (e.g., Chapter 5.4 of \citep[][]{absil2009optimization}), for any $p \in \M$ and any $v \in T_p \M$, 
    \begin{align} 
        \Hess f (p) (v,v) 
        = 
        \lim_{\tau \rightarrow 0 } \frac{ f (\Exp_p (\tau v )) - 2 f ( p ) + f (\Exp_p ( - \tau v )) }{ \tau^2 } 
        = 
        \nabla_v^2 f (p) . \label{eq:Hess-dev} 
    \end{align} 
    \emph{For simplicity and coherence with the notations in the Euclidean case}, we write $ u^\top \Hess f (p) v := \Hess f (p) (u, v) $ for any $u,v \in T_p \M$. 
    
    \item Consider a Riemannian manifold $(\M, g)$, a point $p \in \M$, and any symmetric bilinear form $A : T_p \M  \times T_p \M \rightarrow \R$. The $g$-induced $\infty$-Schatten norm (the operator norm) of $A$ is defined as 
    \begin{align*} 
        \| A \| = \sup_{u \in T_p \M, \| u \| = 1} | u^\top A u |. 
    \end{align*} 
    When it is clear from context, we simply use $ \infty $-Schatten norm to refer to $g$-induced $ \infty $-Schatten norm. 
    
    \item \emph{Note.} When applied to a tangent vector, $\| \cdot \|$ is the standard norm induced by the Riemannian metric. When applied to a symmetric bilinear form, $\| \cdot \|$ is the $ \infty $-Schatten norm defined above. 
    
    
    
\end{itemize}

\section{Zeroth Order Hessian Estimation} 
\label{sec:hess-comp} 

For $p \in \M$ and $f : \M \rightarrow \R$, the Hessian of $f$ at $ p $ can be estimated by 
\begin{align*}
    \widehat{\H} {f} (p; v,w; \delta) = \frac{n^2}{\delta^2} f ( \Exp_p ( \delta v + \delta w) ) v \otimes w ,
\end{align*}
where $ v,w $ are independently uniformly sampled from $\S_p$ and $ \delta $ is the finite difference step size. To study the bias of this estimator, we consider a function $\wt{f}^\delta$ defined as follows. 

For $p \in \M$, a smooth real-valued function $f$ defined over $\M$, and a number $\delta \in (0, \delta_0 ]$, define a function $\wt{f}^\delta$ (at $p$) such that 
\begin{align} 
    \wt{f}^\delta ( p ) = \frac{ 1 }{ \delta^{2n} V_n^2 } \int_{ w \in \delta \B_p } \int_{ v \in \delta \B_p } f ( \Exp_p ( v + w ) ) \, dw \, dv,  \label{eq:def-surrogate} 
\end{align} 
where $V_n$ is the volume of the unit ball in $ T_p \M $ (same as the volume of the unit ball in $\R^n$). Smoothings of this kind have been analytically investigated by Greene and Wu \citep{greene1973subharmonicity,greene1976c,greene1979c}. 
We will first show that $ \Hess \wt{f}^\delta ( p ) = \E_{v,w \overset{i.i.d.}{\sim} \S_p } \[ \wh{\H} {f} (p; v,w; \delta) \] $ in Lemma \ref{lem:surrogate}. Then we derive a bound on $ \left\| \Hess \wt{f}^\delta ( p ) - \Hess {f} ( p ) \right\| $, which gives a bound on $ \left\| \E_{v,w \overset{i.i.d.}{\sim} \S_p } \[ \wh{\H} {f} (p; v,w; \delta) \] - \Hess f (p) \right\| $. 
Henceforth, we will use $ \E_{v,w } $ as a shorthand for $ \E_{v,w \overset{i.i.d.}{\sim} \S_p } $.


\begin{lemma} 
    \label{lem:surrogate} 
    Consider an $n$-dimensional complete analytic Riemannian manifold $(\M,g)$. Consider $p \in \M$, an analytic function $f : \M \rightarrow \R $ and a small number $\delta \in (0, \inj (p) / 2 ]$.
    If $v$ and $w$ are independently randomly sampled from $\S_p$, then it holds that, 
    \begin{align*} 
        \E_{v,w} \[ \wh{\H} f (p; v, w; \delta) \] = \Hess \wt{f}^\delta (p) . 
    \end{align*} 
\end{lemma} 


\begin{proof}[Proof. ]
    Define $ \varphi_p = f \circ \Exp_p $. 
    By the fundamental theorem of geometric calculus, it holds that~\footnote{Here $\partial_i$ and $\partial_j$ are understood as Einstein's notations. } 
    \begin{align*} 
        \int_{v \in \delta \B_p } \partial_i \int_{ w \in \delta \B_p }  \partial_j \varphi_p (w + v) \, dw \, dv 
        =& 
        \int_{v \in \delta \B_p } \partial_i \int_{ w \in \delta \S_p } \varphi_p (w + v) \frac{ w }{ \| w \| } \, dw \, dv \nonumber  \\ 
        \overset{(i)}{=}& 
        \int_{v \in \delta \S_p } \int_{ w \in \delta \S_p } \varphi_p (w + v) \frac{v \otimes w }{ \| w \| \| v \|  } \, dw \, dv . 
    \end{align*} 
    
    Since $v$ and $w$ are independently uniformly sampled from $\S_p$, it holds that 
    \begin{align*} 
        \int_{ \delta \S_p } \int_{ \delta \S_p } \varphi_p ( w + v ) \frac{ v \otimes w }{ \| v \| \| w \| } \, dw \, dv  \overset{(ii)}{=} \delta^{2n-2} A_{n-1}^2 \E_{v,w} \[ \varphi_p ( \delta v + \delta w) v \otimes w \], 
    \end{align*} 
    where $ A_{n-1} $ is the surface area of $\S_p $ in $T_p \M$ (same as the surface area of unit sphere in $\R^n$). 
    
    By the dominated convergence theorem and that $\delta \le \frac{\inj (p)}{2}$, we have 
    \begin{align*} 
        \partial_i \partial_j  \int_{ v \in \delta \B_p } \int_{ w \in \delta \B_p }  \varphi_p (w + v) \, dw \, dv 
        \overset{(iii)}{=}& \; 
        \int_{ v \in \delta \B_p } \partial_i \int_{ w \in \delta \B_p } \partial_j \varphi_p (w + v) \, dw \, dv . 
    \end{align*} 
    More specifically, the $\partial_i$ operations (or tangent vectors) can be defined in terms of limits, and we can interchange the limit and the integral by the dominated convergence theorem. 
    
    Combining $(i)$, $(ii)$ and $(iii)$ gives
    \begin{align*}
        \partial_i \partial_j \int_{ v \in \delta \B_p } \int_{ w \in \delta \B_p }  \varphi_p (w + v) \, dw \, dv 
        \overset{(iv)}{=}& \; 
        \delta^{2n-2} A_{n-1}^2 \E_{v,w} \[ \varphi_p ( \delta v + \delta w) v \otimes w \]. 
    \end{align*}
    
    
    Combining the above results gives 
    \begin{align*} 
        \partial_i \partial_j \wt{f}^\delta ( p ) 
        =& \; 
        \partial_i \partial_j \frac{1}{ \delta^{2n} V_n }\int_{ v \in \delta \B_p } \int_{ w \in \delta \B_p } \varphi_p (w + v) \, dw \, dv \\ 
        =& \; 
        \frac{ \delta^{2n-2} A_{n-1}^2 }{ \delta^{2n} V_n^2 } \E_{v,w} \[ \varphi_p ( \delta v + \delta w) v \otimes w \] \\
        =& \; 
        \frac{ n^2 }{ \delta^2 }  \E_{v,w} \[ f ( \Exp_p ( \delta v + \delta w) ) v \otimes w \] , 
    \end{align*} 
    where the second last equality uses $ (iv) $, and last equality uses $ A_{n-1} = n V_n $. \hfill
\end{proof} 

As a result of Lemma \ref{lem:surrogate}, a bound on $ \| \Hess \wt{f}^\delta (p) - \Hess f (p) \| $ will give a bound on $ \| \E \[ \wh{\H} f (p; v,w; \delta) \] - \Hess f (p) \| $. To bound $ \| \Hess \wt{f}^\delta (p) - \Hess f (p) \| $, we need to explicitly extend the definition of $ \wt{f}^\delta $ from $p$ to a neighborhood of $p$ \citep{wang2021GW}, so that the Hessian can be computed in a precise manner. For $p \in \M$, a smooth function $f: \M \rightarrow \R$, and a number $\delta \in (0,\delta_0]$, define a function $\wt{f}^\delta$ (near $p$) such that 
\begin{align} 
    \wt{f}^\delta (q) = \E_{v,w \overset{i.i.d.}{\sim} \S_p } \[ \wt{f}_{v,w}^\delta (q) \], \quad \forall q \in U_p, \label{eq:def-surrogate2}
\end{align} 
where 
\begin{align} 
    \wt{f}_{v,w}^\delta (q) 
    := 
    \frac{ n^2 }{ 4 \delta^{2n} } \int_{-\delta}^\delta \int_{-\delta}^\delta f \( \Exp_q (t v_q + s w_q) \) |t|^{n-1} |s|^{n-1} \, dt \, ds , \label{eq:def-fvw} 
\end{align} 
with $ v,w \in \S_p$. 

The advantage of defining $ \wt{f}^\delta  $ via $\wt{f}_{v,w}^\delta$ is that $ \wt{f}_{v,w}^\delta $ is explicitly defined in a neighborhood of $p$. Thus we can carry out geometry-aware computations in a precise manner. Next, we verify that (\ref{eq:def-surrogate}) and (\ref{eq:def-surrogate2}) agree with each other in the following proposition.

\begin{proposition}
    \label{prop:sur1-sur2}
    For any $ p \in \M $ and any $ \delta \le \delta_0 $, 
    (\ref{eq:def-surrogate}) and (\ref{eq:def-surrogate2}) coincide at any $ q \in U_p $. 
\end{proposition}

\begin{proof}[Proof. ]
    At any $q \in U_p$, we have 
    \begin{align*} 
        (\ref{eq:def-surrogate}) 
        =& \; 
        \frac{ 1 }{ \delta^{2n} V_n^2 } \int_{ w \in \delta \B_q } \int_{ v \in \delta \B_q } f ( \Exp_q ( v + w ) ) \, dv \, dw \\ 
        \overset{(i)}{=}& \;  
        \frac{ n^2 }{ \delta^{2n} A_n^2 } \int_{ w \in \delta \B_q } \int_{ v \in \delta \B_q } f ( \Exp_q ( v + w ) ) \, dv \, dw \\ 
        \overset{(ii)}{=}& \;  
        \frac{ n^2 }{ \delta^{2n} A_n^2 } \int_{ w \in \S_q } \int_{ v \in \S_q } \frac{1}{4} \int_{-\delta }^\delta \int_{-\delta}^\delta f ( \Exp_q ( t v + s w ) ) |t|^{n-1} |s|^{n-1} \, dt \,ds \, dv \, dw ,  
    \end{align*} 
    where $(i)$ uses $A_{n-1} = n V_n $, and $(ii)$ changes from Cartesian coordinate to hyperspherical coordinate (in $T_q \M$). Since the Levi-Civita connection is compatible with the Riemannian metric, we know that the standard Lebesgue measure in $T_p \M$ is preserved after transporting to $ T_q \M $. This implies, for any continuous function $h$ defined over $ T_q \M $, we have $ \int_{v \in \S_q} h (v) dv \overset{(iii)}{=} \int_{v \in \S_p} h (v_q) dv $. 
    Thus we have, at any $q \in \M$, 
    \begin{align*} 
        (\ref{eq:def-surrogate}) 
        =&\; 
        \frac{ n^2 }{ 4 \delta^{2n} A_n^2 } \int_{ w \in \S_q } \int_{ v \in \S_q } \int_{-\delta }^\delta \int_{-\delta}^\delta f ( \Exp_q ( t v + s w ) ) |t|^{n-1} |s|^{n-1} \, dt \,ds \, dw \, dv \\
        \overset{(iv)}{=}& \;
        \frac{ n^2 }{ 4 \delta^{2n} A_n^2 } 
        \int_{ w \in \S_p } \int_{ v \in \S_p }  \int_{-\delta }^\delta \int_{-\delta}^\delta f ( \Exp_q ( t v_q + s w_q ) ) |t|^{n-1} |s|^{n-1} \, dt \,ds \, dw \, dv \\
        =& \; 
        \frac{1}{A_n^2 } \int_{ w \in \S_p } \int_{ v \in \S_p } \wt{f}_{v,w}^\delta (q) \, dw \, dv  = (\ref{eq:def-surrogate2}) , 
    \end{align*} 
    where $(iv)$ uses $(iii)$. \hfill
\end{proof}

By Proposition \ref{prop:sur1-sur2}, it is sufficient to work with $ \wt{f}_{v,w}^\delta $ and randomize $v,w$ over a unit sphere. For any direction $ u \in \S_p $, the Hessian of $ \wt{f}_{v,w}^\delta $ along $u$ can be explicitly written out in terms of $f$ and $u,v,w$. This result is found in Lemma \ref{lem:est-H}. 

\begin{lemma} 
    \label{lem:est-H} 
    Consider an $n$-dimensional complete analytic Riemannian manifold $(\M,g)$. Consider $p \in \M$, an analytic function $f : \M \rightarrow \R $ and a small number $\delta \in (0, \inj (p) / 2 ]$.
    For any $u, v, w \in \S_p $, we have 
    \begin{align*} 
        &\; u^\top \Hess \wt{f}_{v,w}^\delta (p) u \\
        =& \; 
        \frac{ n^2 }{ 4 \delta^{2n} }  \int_{-\delta}^\delta \int_{-\delta}^\delta u_{q}^\top \Hess f ( q ) u_q \;  |t|^{n-1} |s|^{n-1} \, dt \, ds \\ 
        & + \frac{ n^2 }{ 4 \delta^{2n} }  \int_{-\delta}^\delta \int_{-\delta}^\delta  \sum_{j=1}^\infty \frac{ |t|^{n-1} |s|^{n-1} }{ (2j)! } \nabla_{ u}^2 \( t \nabla_v + s \nabla_w \)^{2j} f (p) \; dt \, ds \nonumber \\ 
        & - \frac{ n^2 }{ 4 \delta^{2n} }  \int_{-\delta}^\delta \int_{-\delta}^\delta  \sum_{j=1}^\infty \frac{ |t|^{n-1} |s|^{n-1} }{ (2j)! } \( t \nabla_v + s \nabla_w \)^{2j} \nabla_{ u}^2 f (p) \; dt \, ds, 
    \end{align*} 
    where $q = \Exp_p (tv+sw)$. 
\end{lemma}

\begin{proof}[Proof. ]
    
    
    
    
    From the definition of Hessian, we have 
    \begin{align*}
        &\; u^\top \Hess \wt{f}_{v,w}^\delta (p) u \\ 
        =& \; 
        \lim_{\tau \rightarrow 0} \frac{ \wt{f}_{v,w}^\delta ( \Exp_p (\tau u) ) - 2 \wt{f}_{v,w}^\delta (p) + \wt{f}_{v,w}^\delta ( \Exp_p ( - \tau u) ) }{ \tau^2 } . 
    \end{align*}
    
    Thus it is sufficient to fix any $t,s \in [-\delta, \delta]$ and consider 
    \begin{align*} 
        \lim_{\tau \rightarrow 0} \frac{ f (\Exp_p^2 ( \tau u, t v + s w  )) - 2 f (\Exp_p ( t v + s w )) + f (\Exp_p^2 ( - \tau u, t v + s w ))  }{ \tau^2 } .
    \end{align*} 

    For simplicity, define 
    \begin{align*} 
        \phi ( \tau, t, s ) 
        =& \; 
        f (\Exp_p^2 ( \tau u, t v + s w  ))  + f (\Exp_p^2 ( - \tau u, t v + s w )) \\
        & - f (\Exp_p^2 ( t v + s w, \tau u  )) - f (\Exp_p^2 ( t v + s w, -\tau u )) . 
    \end{align*} 
    Let $q = \Exp_p (tv+sw)$, and we have 
    \begin{align} 
         &\; u^\top \Hess \wt{f}_{v,w}^\delta (p) u \nonumber \\
        =& \;  
        \frac{ n^2 }{ 4 \delta^{2n} } \int_{-\delta}^\delta \int_{-\delta}^\delta u_{q}^\top \Hess f ( q ) u_{q} \;  |t|^{n-1} |s|^{n-1} \, dt \, ds \nonumber \\ 
        & + 
        \frac{ n^2 }{ 4 \delta^{2n} } \lim_{\tau \rightarrow 0} \frac{ \int_{ -\delta }^\delta \int_{-\delta}^\delta \phi (\tau, t, s) |t|^{n-1} |s|^{n-1} \, dt \, ds }{ \tau^2 } , \label{eq:change-order} 
    \end{align} 
    provided that the last term converges. 
    

    For any $p \in \M$, $v \in T_p \M$ and $q \in U_p$, define $h_v^{(j)} (q) = \nabla_{v_q}^j f (q)$. We can Taylor expand $ h_v^{(j)} (\Exp_p (u) ) $ by 
    \begin{align*} 
        h_v^{(j)} (\Exp_p (u) ) 
        =& \; 
        h_v^{(j)} (\Exp_p ( t u) ) \big|_{t=1} \\
        =& \; 
        \sum_{i=0}^\infty \frac{1}{i!} \frac{d^i}{ dt^i } h_v^{(j)} (\Exp_p ( t u) ) \bigg|_{t=0} \\
        =& \; 
        \sum_{i=0}^\infty \frac{1}{i!} \nabla_u^i h_v^{(j)} ( p ) \\
        \overset{(a)}{=}& \; 
        \sum_{i=0}^\infty \frac{1}{i!} \nabla_u^i \nabla_v^j f ( p ) .  
    \end{align*} 
    From above, we have, for any $p$, and $u,v \in T_p\M$ of small norm, 
    \begin{align*} 
        f \( \Exp_{ p }^2 \( u,v \) \) 
        =& \;
        f \( \Exp_{ \Exp_p (u) } \( v_{\Exp_p (u)} \) \) \\
        =& \;
        \sum_{j=0}^\infty \frac{1}{j!} \nabla_{ v_{\Exp_p (u)} }^j f ( \Exp_p (u) ) \\
        =& \;
        \sum_{j=0}^\infty \frac{1}{j!} h_{v}^{(j)} (\Exp_p (u) ) \\
        =& \; 
        \sum_{j=0}^\infty \frac{1}{j!} \sum_{i=0}^\infty \frac{1}{i!} \nabla_u^i \nabla_v^j f ( p ) , 
    \end{align*} 
    where the second equality uses Taylor expansion at $\Exp_p (u)$ and the last equality uses $ (a) $. 
    
    
    From the above computation, we expand $ f ( \Exp_p^2 ( t v + s w, \tau u ) )  $ (and similar terms) into infinite series. 
    %
    %
    %
    Thus we can write $\phi (\tau, t,s)$ as 
    \begin{align} 
        \phi (\tau, t,s) 
        =&\; f (\Exp_p^2 ( \tau u, t v + s w  ))  + f (\Exp_p^2 ( - \tau u, t v + s w )) \nonumber \\
        & - f (\Exp_p^2 ( t v + s w, \tau u  )) - f (\Exp_p^2 ( t v + s w, -\tau u )) \nonumber \\ 
        =& \; 
        \sum_{i=0}^\infty \sum_{j=0}^\infty  \frac{1}{i! j! } \nabla_{\tau u}^i \nabla_{t v + sw}^j  f (p) + \sum_{i=0}^\infty \sum_{j=0}^\infty  \frac{1}{i! j! } \nabla_{ - \tau u}^i \nabla_{t v + s w }^j  f (p) \nonumber \\ 
        &- \sum_{i=0}^\infty \sum_{j=0}^\infty \frac{1}{i! j! }  \nabla_{t v + sw}^j \nabla_{ \tau u}^i f (p) - \sum_{i=0}^\infty \sum_{j=0}^\infty  \frac{1}{i! j! }  \nabla_{t v + sw}^j \nabla_{ - \tau u}^i f (p) \nonumber \\ 
        =& \; 
        \sum_{j=0}^\infty  \frac{ \tau^2 }{2 j! } \nabla_{ u}^2 \nabla_{t v + sw}^j f (p) + \sum_{j=0}^\infty \frac{ \tau^2 }{2 j! } \nabla_{ u}^2 \nabla_{t v + sw}^j  f (p) \nonumber \\ 
        &- \sum_{j=0}^\infty  \frac{ \tau^2 }{2 j! } \nabla_{t v + sw}^j \nabla_{ u}^2 f (p) - \sum_{j=0}^\infty  \frac{ \tau^2 }{2 j! } \nabla_{t v + sw}^j \nabla_{ u}^2 f (p) + {O} (\tau^3) , \label{eq:series1} 
    \end{align} 
    where the last equality uses that zeroth-order terms in $\tau$ and first-order terms in $\tau$ all cancel. 

    From (\ref{eq:series1}), we have 
    \begin{align*}
        \lim_{\tau \rightarrow 0 } \frac{\phi ( \tau, t, s )}{ \tau^2 } 
        =& \;  
        \sum_{j=1}^\infty  \frac{ 1 }{j! } \nabla_{ u}^2 \nabla_{t v + sw}^j f (p)  
        - \sum_{j=1}^\infty  \frac{ 1 }{j! } \nabla_{t v + sw}^j \nabla_{ u}^2 f (p)  \\
        =& \;  
        \sum_{j=1}^\infty  \frac{ 1 }{j! } \nabla_{ u}^2 \( t \nabla_v + s \nabla_w \) ^j f (p) 
        - \sum_{j=1}^\infty  \frac{ 1 }{j! } \( t \nabla_v + s \nabla_w \)^j \nabla_{ u}^2 f (p) \\
        =& \; 
        \sum_{j=1}^\infty  \frac{ 1 }{ (2j)! } \nabla_{ u}^2 \( t \nabla_v + s \nabla_w \)^{2j} f (p) \\
        &- \sum_{j=1}^\infty  \frac{ 1 }{ (2j)! } \( t \nabla_v + s \nabla_w \)^{2j} \nabla_{ u}^2 f (p) + Odd (t,s) , 
    \end{align*} 
    where $Odd (t,s) $ denotes terms that are either odd in $t$ or odd in $s$. 

    
    
    Since $ \int_{-\delta}^\delta \int_{-\delta}^\delta Odd (t,s) \, dt \, ds = 0 $, we have 
    \begin{align} 
        & \; \frac{ n^2 }{ 4 \delta^{2n} } \lim_{\tau \rightarrow 0} \frac{ \int_{ -\delta }^\delta \int_{-\delta}^\delta \phi (\tau, t, s) |t|^{n-1} |s|^{n-1} \, dt \, ds }{ \tau^2 } \nonumber \\ 
        =& \; 
        \frac{ n^2 }{ 4 \delta^{2n} }  \int_{-\delta}^\delta \int_{-\delta}^\delta  \sum_{j=1}^\infty \frac{ |t|^{n-1} |s|^{n-1} }{ (2j)! } \nabla_{ u}^2 \( t \nabla_v + s \nabla_w \)^{2j} f (p) \; dt \, ds \nonumber \\ 
        & - \frac{ n^2 }{ 4 \delta^{2n} }  \int_{-\delta}^\delta \int_{-\delta}^\delta  \sum_{j=1}^\infty \frac{ |t|^{n-1} |s|^{n-1} }{ (2j)! } \( t \nabla_v + s \nabla_w \)^{2j} \nabla_{ u}^2 f (p) \; dt \, ds . \label{eq:series2} 
    \end{align}  
    
    Collecting terms from (\ref{eq:change-order}) and (\ref{eq:series2}), we have 
    \begin{align*} 
        &\; u^\top \Hess \wt{f}_{v,w}^\delta (p)  u \\
        =& \; 
        \frac{ n^2 }{ 4 \delta^{2n} }  \int_{-\delta}^\delta \int_{-\delta}^\delta u_{q}^\top \Hess f ( q ) u_{q} \;  |t|^{n-1} |s|^{n-1} \, dt \, ds \\ 
        & +  \frac{ n^2 }{ 4 \delta^{2n} }  \int_{-\delta}^\delta \int_{-\delta}^\delta  \sum_{j=1}^\infty \frac{ |t|^{n-1} |s|^{n-1} }{ (2j)! } \nabla_{ u}^2 \( t \nabla_v + s \nabla_w \)^{2j} f (p) \; dt \, ds \nonumber \\ 
        & - \frac{ n^2 }{ 4 \delta^{2n} }  \int_{-\delta}^\delta \int_{-\delta}^\delta  \sum_{j=1}^\infty \frac{ |t|^{n-1} |s|^{n-1} }{ (2j)! } \( t \nabla_v + s \nabla_w \)^{2j} \nabla_{ u}^2 f (p) \; dt \, ds ,
    \end{align*} 
    where $ q = \Exp_p (tv+sw) $. This concludes the proof. \hfill
    \end{proof}  

Gathering the above results gives a bias bound for (\ref{eq:def-est}), which is summarized in the following theorem. 
\begin{theorem} 
    \label{thm:hess-bias}
    Consider an $n$-dimensional complete analytic Riemannian manifold $(\M,g)$. Consider $p \in \M$, an analytic function $f : \M \rightarrow \R $ and a small number $\delta \in (0, \inj (p) / 2 ]$. 
    For any $p\in \M$ and unit vectors $u,v \in T_p \M$, define a function $ \vartheta_{p,u,v,w} $ over $(-\inj (p) / 2, \inj (p) / 2) \times (- \inj (p) / 2, \inj (p) / 2)$ such that 
    \begin{align*} 
        \vartheta_{p,u,v,w} (t,s) 
        = 
        \Hess f (\Exp_p(tv+sw)) (u_{\Exp_p(tv+sw)}, u_{\Exp_p(tv+sw)}). 
    \end{align*} 
    The estimator (\ref{eq:def-est}) satisfies 
    \begin{align*} 
        & \; \left\| \E_{v,w} \[ \wh{\H} f (p; v,w; \delta) \] - \Hess {f} (p) \right\| \\ 
        \le& \; 
        \sup_{u \in \S_p} \Bigg| 
        \E_{v,w} \[ \frac{ n^2 }{ 4 \delta^{2n} } \int_{-\delta}^\delta \int_{-\delta}^\delta  \sum_{i,j \in \mathbb{N}, i+j \ge 1} \frac{t^{2i} s^{2j}}{(2i)! (2j)! } \partial_1^{2i} \partial_2^{2j} \vartheta_{p,u,v,w} (0,0)  |t|^{n-1} |s|^{n-1} \, dt \, ds \] \\
        &+ 
        \E_{v,w} \Bigg[ \frac{ n^2 }{ 4 \delta^{2n} }  \int_{-\delta}^\delta \int_{-\delta}^\delta  \sum_{j=1}^\infty \frac{ |t|^{n-1} |s|^{n-1} }{ (2j)! } \nabla_{ u}^2 \( t \nabla_v + s \nabla_w \)^{2j} f (p) \; dt \, ds \nonumber \\ 
        &- 
        \frac{ n^2 }{ 4 \delta^{2n} }  \int_{-\delta}^\delta \int_{-\delta}^\delta  \sum_{j=1}^\infty \frac{ |t|^{n-1} |s|^{n-1} }{ (2j)! } \( t \nabla_v + s \nabla_w \)^{2j} \nabla_{ u}^2 f (p) \; dt \, ds \Bigg] \Bigg| ,
    \end{align*}
    where $v,w$ are independently sampled from the uniform distribution over $ \S_p $. 
\end{theorem} 

\begin{proof}[Proof. ] 
    By Lemma \ref{lem:est-H}, we have 
    \begin{align*} 
        & \; u^\top \E_{v,w} \[ \Hess \wt{f}_{v,w}^\delta (p) \] u \\
        =& \;  
        \frac{ n^2 }{ 4 \delta^{2n} } \E_{v,w} \[ \int_{-\delta}^\delta \int_{-\delta}^\delta u_{\Exp_p (tv+sw)}^\top \Hess f ( \Exp_p (tv + sw) ) u_{\Exp_p (tv + sw)} \; |t|^{n-1} |s|^{n-1} \, dt \, ds \] \\
        &+ 
        \frac{ n^2 }{ 4 \delta^{2n} } \E_{v,w} \[  \int_{-\delta}^\delta \int_{-\delta}^\delta  \sum_{j=1}^\infty \frac{ |t|^{n-1} |s|^{n-1} }{ (2j)! } \nabla_{ u}^2 \( t \nabla_v + s \nabla_w \)^{2j} f (p) \; dt \, ds \] \nonumber \\ 
        &- 
        \frac{ n^2 }{ 4 \delta^{2n} } \E_{v,w} \[ \int_{-\delta}^\delta \int_{-\delta}^\delta  \sum_{j=1}^\infty \frac{ |t|^{n-1} |s|^{n-1} }{ (2j)! } \( t \nabla_v + s \nabla_w \)^{2j} \nabla_{ u}^2 f (p) \; dt \, ds \] . 
    \end{align*} 
    
    By the analyticity assumption, we have 
    \begin{align*} 
        \vartheta_{p,u,v,w} (t,s) 
        =& \;  
        \sum_{i=0}^\infty \sum_{j=0}^\infty \frac{t^i s^j}{i! j! } \partial_1^i \partial_2^j \vartheta_{p,u,v,w} (0,0) . 
    \end{align*} 
    
    For any fixed $u \in \S_p $, we have 
    \begin{align*} 
        &\; \frac{ n^2 }{ 4 \delta^{2n} }  \E_{v,w} \[ \int_{-\delta}^\delta \int_{-\delta}^\delta u_{\Exp_p (tv+ sw)}^\top \Hess f ( \Exp_p (tv+ sw) ) u_{\Exp_p (tv+ sw)} \;  |t|^{n-1} |s|^{n-1} \, dt \, ds \] \\
        =&\; 
        \frac{ n^2 }{ 4 \delta^{2n} }  \E_{v,w} \[ \int_{-\delta}^\delta \int_{-\delta}^\delta \vartheta_{p,u,v,w} (t,s) \;  |t|^{n-1} |s|^{n-1} \, dt \, ds \]  \\ 
        =&\; 
        \frac{ n^2 }{ 4 \delta^{2n} } \E_{v,w} \[ \int_{-\delta}^\delta \int_{-\delta}^\delta \( \sum_{i=0}^\infty \sum_{j=0}^\infty \frac{t^i s^j}{i! j! } \partial_1^i \partial_2^j \vartheta_{p,u,v,w} (0,0)  \) |t|^{n-1} |s|^{n-1} \, dt \, ds \] \\
        =& \; 
        \frac{ n^2 }{ 4 \delta^{2n} } \E_{v,w} \[ \int_{-\delta}^\delta \int_{-\delta}^\delta  \sum_{i,j \in \mathbb{N}, i+j \ge 1} \frac{t^{2i} s^{2j}}{(2i)! (2j)! } \partial_1^{2i} \partial_2^{2j} \vartheta_{p,u,v,w} (0,0)  |t|^{n-1} |s|^{n-1} \, dt \, ds \] + u^\top \Hess f (p) u , 
    \end{align*} 
    where in the last line the terms that are odd in $t$ or $s$ vanishes. \hfill
\end{proof} 


By applying (\ref{eq:Hess-dev}) twice, we have 
\begin{align*}
    \partial_1^{2} \vartheta_{p,u,v,w} (0,0) 
    =
    \nabla_v^2 \nabla_u^2 f (p)
    &\qquad \mathrm{and} \qquad 
    \partial_2^{2} \vartheta_{p,u,v,w} (0,0) 
    =
    \nabla_w^2 \nabla_u^2 f (p) . 
\end{align*} 

Thus by dropping terms of order $O (\delta^3)$ and noting that $ \int_{-\delta}^\delta \int_{-\delta}^\delta |t|^{n-1} t |s|^{n-1} s \; dt \; ds = 0 $, we have 
\begin{align*} 
    &\; \left\| \E_{ v,w} \[ \wh{\H} f (p; v,w; \delta ) \] - \Hess f (p) \right\| \\ 
    =& \; 
    O \( \delta^2 \sup_{u \in \S_p} \left| \E_{ v,w \overset{i.i.d.}{\sim} \S_p } \[ \frac{n}{n+2} \nabla_{ u}^2 \( \nabla_v^2 + \nabla_w^2 \) f (p) \] \right| \) . 
\end{align*} 

    

\subsection{Example: the $n$-sphere} 
We consider the Riemannian manifold $ \S^{n-1} $ with metric induced by the ambient Euclidean space. This space of both theoretical and practical appeal. In this space, the exponential map is 
\begin{align*}
    \Exp_x ( t v ) 
    = 
    x \cos(t) + v \sin( t ), 
\end{align*}
where $v$ is a unit vector in $\R^n$; The parallel transport is 
\begin{align*}
    \I_x^{\Exp_x(tu)} (v)
    = 
    v - u u^\top v + u u^\top v \cos (t) - \| u u^\top v \| x \sin (t) 
\end{align*} 
for any $x \in \S^{n-1}$, $u,v \in \S^{n-1} $ and $u,v \perp x$. 

We will consider estimating Hessian of the function $ x_i^2 $ where $x \in \S^{n-1}$ and $ x_i $ is the $i$-th component of $x$. This simple function serves as an example of estimating the Hessian of general polynomials over $\S^{n-1}$. 

We have 
\begin{align*}
    \nabla_v^2 x_i^2 
    =& \;  
    \lim_{t \rightarrow 0} \frac{ \( x_i \cos t + v_i \sin t \)^2 - 2 x_i^2 + \( x_i \cos t - v_i \sin t \)^2 }{t^2} \\ 
    =& \; 
    \lim_{t \rightarrow 0} \frac{ \( 2 \cos^2 t - 2 \) x_i^2 + 2 v_i^2 \sin^2 t }{t^2} \\ 
    =& \; 
    - 2 x_i^2 + 2 v_i^2 ,
\end{align*} 
and 
\begin{align*}
    &\; \nabla_u^2 v_i^2 \\ 
    =& \; 
    \lim_{t \rightarrow 0} \frac{ 2 \( v_i - \( u u^\top v \)_i + \( u u^\top v \)_i \cos t \)^2 + 2\( \| u u^\top v \| x_i \sin t \)^2 - 2 v_i^2 }{t^2} \\ 
    =& \; 
    \lim_{t \rightarrow 0} \frac{ 4 v_i \( uu^\top v \)_i (\cos t - 1 ) + 2 (u u^\top v)_i^2 \( \cos t - 1 \)^2 + 2\( \| u u^\top v \| x_i \sin t \)^2  }{t^2} \\ 
    =& \; 
    -2 v_i \( u u^\top v \)_i + 2 \| u u ^\top v \|^2 x_i^2. 
\end{align*} 
Thus it holds that 
\begin{align*} 
    \nabla_u^2 \nabla_v^2 x_i^2 
    =& \; 
    \nabla_u^2 \( - 2 x_i^2 + 2 v_i^2 \) \\ 
    =& \; 
    - 2 \( - 2 x_i^2 + 2 u_i^2 \) 
    + 
    2 \( -2 v_i \( u u^\top v \)_i + 2 \| u u ^\top v \|^2 x_i^2  \) \\ 
    =& \; 
    4 x_i^2 - 4 u_i^2 - 4 v_i \( u u^\top v \)_i + 4 \| u u ^\top v \|^2 x_i^2 . 
\end{align*} 
Since $ \E_{v} \[ v v^\top \] = \frac{1}{n} I $ and $\E_{w} \[ w w^\top \] = \frac{1}{n} I$, we have 
\begin{align*} 
    \E_{v,w} \[ \nabla_u^2 \( \nabla_v^2 + \nabla_w^2 \) x_i^2  \] 
    =& \;  
    2\( 4 x_i^2 - 4 u_i^2 - \frac{4}{n} u_i^2 + \frac{4}{n} x_i^2 \) \\
    =& \; 
    \( 8 + \frac{8}{n} \) x_i^2 - \( 8 + \frac{8}{n} \) u_i^2. 
\end{align*} 
This implies that the Hessian estimator for $x_i^2$ over the $n$-sphere with granularity $\delta$ is of order 
\begin{align*}
    O \( \delta^2 \max_{u \in \S^{n-1}, u \perp x } \left| \( 8 + \frac{8}{n} \) x_i^2 - \( 8 + \frac{8}{n} \) u_i^2 \right| \) . 
\end{align*}

\section{The Euclidean Case}

In this section, we will focus on numerical stabilization of the estimation, and algorithmic zeroth-order inversion of the estimated Hessian. For numerical and algorithmic purposes, we restrict our attention to the Euclidean case. In the Euclidean case, we also use the notation $\nabla^2 f (x)$ to denote the Hessian of $f$ at $x$. 

\subsection{Stabilizing the Estimate}

In the Euclidean case, the estimator in (\ref{eq:def-est}) simplifies to 
\begin{align} 
    \wh{\H} {f} (p; v,w; \delta) = \frac{n^2}{\delta^2} f (p + \delta v + \delta w) v w^\top , \nonumber
\end{align} 
where $ v,w$ are independently uniformly sampled from $\S^{n-1}$ (the unit sphere in $\R^n$). Its stabilized version is 
\begin{align} 
    \wh{\H} {f} ( p; v,w; \delta ) 
    =& \;  
    \frac{n^2}{8 \delta^2} \big[ f (p + \delta v + \delta w) - f (p - \delta v + \delta w) \nonumber \\
    &- f (p + \delta v - \delta w) + f (p - \delta v - \delta w) \big] \( v w^\top + w v^\top \) \label{eq:def-euc-stab} . 
\end{align}

To see why (\ref{eq:def-euc-stab}) stabilizes the estimate, we use Taylor expansion and get 
\begin{align}  
    & \; f (p + \delta v + \delta w) - f (p - \delta v + \delta w) - f (p + \delta v - \delta w) + f (p - \delta v - \delta w) \nonumber \\ 
    \approx& \; 
    {\delta^2} \( v + w \)^\top \nabla^2 f (x) \( v + w \) - \frac{\delta^2}{2} \( v - w \)^\top \nabla^2 f (x) \( v - w \) \nonumber \\
    &- \frac{\delta^2}{2} \( - v + w \)^\top \nabla^2 f (x) \( - v + w \) \nonumber \\ 
    {=}& \; 
    4 {\delta^2} v^\top \nabla^2 f (x) w , \label{eq:Taylor-cancel}
\end{align}  
where $\nabla^2 f $ denotes the Hessian of $f$. 

From the above derivation, we see that (\ref{eq:def-euc-stab}) removes the dependence on the zeroth-order and first-order information, and symmetrizes the estimation \citep{fengwang2022}. This can reduce variance and stabilize the estimation. A similar phenomenon for the gradient estimators is noted by \citet{duchi2015optimal}. 

\subsubsection{A Random Projection Derivation}

Similar to gradient estimators \citep{nesterov2017random,li2020stochastic,wang2021GW,fengwang2022}, one may also derive the Hessian estimator (\ref{eq:def-euc-stab}) using a random projection argument. 
Here we use the spherical random projection argument to derive the Hessian estimator. A more thorough study can be found in \citep{fengwang2022}. 
To start with, we first prove an identity for random matrix projection in Lemma \ref{lem:mat-rand-proj}. 

\begin{lemma}
    \label{lem:mat-rand-proj}
    Let $v,w$ be independently uniformly sampled from the unit sphere in $\R^n$. For any matrix $A \in \R^{n \times n}$, we have 
    \begin{align*} 
        \E \[ \( v^\top A w \) w v^\top \] = \frac{1}{n^2} A. 
    \end{align*} 
\end{lemma} 

\begin{proof}[Proof. ] 
    It is sufficient to show $ \E \[ v^i A_i^j w_j v^k w_l \] = \frac{1}{n^2} A_l^k $ for any $k,l \in [n]$ (Einstein's notation is used). 
    
    Since $ v $ is uniformly sampled from $\S^{n-1}$ (the unit sphere in $\R^n$), for $k\neq i$, we have $ \E \[ v^i v^k | v^k = x \] = 0 $ for any $x$. This gives that 
    \begin{align*} 
        \E \[ v^i v^k \] \overset{(i)}{=} \int_{x \in \[ -1, 1 \] } \Pr \( v^k = x \) \E \[ v^i v^k | v^k = x \] \; dx = 0, \qquad \forall k \neq i. 
    \end{align*} 
    
    By symmetry of the sphere $\S^{n-1}$ and that $\E \[ v^i v_i \] = 1$, we have $ \E \[ v^k v^k \] \overset{(ii)}{=} \frac{1}{n} $ for any $k \in [n]$. Combining $(i)$ and $(ii)$ gives 
    \begin{align*} 
        \E \[ v^i v^k \] \overset{(iii)}{=} \frac{1}{n} \delta^{ki}, 
    \end{align*} 
    where $ \delta^{ki} $ is the Kronecker's delta with two superscript. 
    
    Similarly, it holds that  $ \E \[ w_j w_l \] \overset{(iv)}{=} \frac{1}{n}\delta_{jl} $, where $ \delta_{jl} $ is the Kronecker's delta with two subscript. Since $v$ and $w$ are independent, $(iii)$ and $(iv)$ gives  
    \begin{align*} 
        \E \[  v^i A_i^j w_j v^k w_l \] = \frac{1}{n^2} A_i^j \delta^{ik} \delta_{jl} = \frac{1}{n^2} A_l^k, 
    \end{align*} 
    which concludes the proof. \hfill
\end{proof} 

With Lemma \ref{lem:mat-rand-proj}, we can see that the estimator in (\ref{eq:def-euc-stab}) satisfies 
\begin{align}
    &\; \E \[ \wh{\H} f (p; v,w; \delta) \] \nonumber \\
    \overset{(i)}{\approx}& \;  
    \frac{ n^2 }{8 \delta^2} \E \[ 4 \delta^2 \( v^\top \nabla^2 f (x) w \) \( v w^\top + w v^\top  \) \] \nonumber \\
    =& \; 
    \frac{n^2}{2} \( \E \[ \( v^\top \nabla^2 f (x) w \)  w v^\top \] + \E \[ \( w^\top \[ \nabla^2 f (x) \]^\top v \) v w^\top \] \) \nonumber \\
    \overset{(ii)}{=}& \;
    \frac{ 1 }{ 2 } \nabla^2 f (x) + \frac{ 1 }{ 2 } \[ \nabla^2 f (x) \]^\top \nonumber \\
    =& \; 
    \nabla^2 f (x) ,
\end{align} 
where $ (i) $ uses (\ref{eq:Taylor-cancel}), and $(ii)$ uses Lemma \ref{lem:mat-rand-proj}. The above argument gives a random-projection derivation for the estimator (\ref{eq:def-euc-stab}).



Similar to \citep{fengwang2022}, we can obtain an $O (\delta^2)$ bias bound in Euclidean spaces. 
\begin{corollary}
    \label{cor:delta2}
    Let $f: \R^n \rightarrow \R$ be a smooth function, and let $\partial^k f$ ($k\in \mathbb{N}_+$) denote the $k$-th order total derivative of $f$. 
    Let $f$ be 4-th order continuously differentiable. Let there be a constant $L_4$ such that $\| \partial^4 f (x) \| \le L_4$ for all $x \in \R^n$, where $ \| \cdot \| $ denotes the spectral norm ($\infty$-Schatten norm) of a tensor. Then it holds that  
    \begin{align*}
        \left\| \E_{v,w } \wh{\H} f (p; v,w; \delta) - \Hess f (p) \right\| 
        \le 
        \frac{ L_4 n \delta^2 }{ n+2 }, \quad \forall p \in \R^n, \delta \in (0,\infty), 
    \end{align*}
    where $ v,w $ are uniformly sampled from the unit sphere $ \S^{n-1} $. 
\end{corollary}
\begin{proof}[Proof. ]
    Firstly, note that the injectivity radius of Euclidean spaces is infinite. 
    By Lemmas \ref{lem:surrogate} and \ref{lem:est-H}, it suffices to consider $  \| \Hess \wh{f}^\delta (p) - \Hess f (p) \| $. In the Euclidean case, for any $u,v,p \in \R^n$, Taylor's theorem gives
    \begin{align*}
        \Hess f (p + v) (u,u) 
        = 
        \Hess f (p) (u,u) + \partial^3 f (p) (u,u,v) + \frac{1}{2} \partial^4 f (p') (u,u,v,v) 
    \end{align*}
    for some $p'$ depending on $p, u,v$. 
    Fix any unit vector $u \in \R^n$, we have 
    \begin{align*}
        &\; \Hess \wt{f}^\delta (p) (u,u) - \Hess f (p) (u,u) \\
        =& \;  
        \frac{1}{\delta^{2n} V_n^2} \int_{v \in \delta \B^n } \int_{w \in \delta \B^n } \Hess f (p + v + w) (u,u) \, d v \, d w 
        - 
        \Hess f (p) (u,u) \tag{$\delta \B^n$ is the origin-centered ball with radius $\delta$.}  \\
        =& \; 
        \frac{1}{\delta^{2n} V_n^2} \int_{v \in \delta \B^n } \int_{w \in \delta \B^n } \( \partial^3 f (p) (u,u,v+w) + \frac{1}{2} \partial^4 f (p') (u,u,v+w,v+w) \) \, d v \, d w \\
        =& \; 
        \frac{1}{2 \delta^{2n} V_n^2} \int_{v \in \delta \B^n } \int_{w \in \delta \B^n }  \partial^4 f (p') (u,u,v+w,v+w) \, d v \, d w . \tag{by symmetry of the ball $\delta \B^n$} 
    \end{align*} 
    By symmetry of $ \delta \B^n $, we know that $ \int_{v \in \delta \B^n } \int_{w \in \delta \B^n }  \partial^4 f (p') (u,u,v,w) \, d v \, d w  = 0 $. 
    Since $\| \partial^4 f (p) \| \le L_4 $ for all $p \in \R^n$, we have that 
    \begin{align*}
        \left|  \int_{v \in \delta \B^n } \int_{w \in \delta \B^n }  \partial^4 f (p') (u,u,v+w,v+w) \, d v \, d w  \right| 
        \le 
        L_4 \frac{2}{n (n+2)} A_n^2 \delta^{2n+2} , 
    \end{align*}
    where $A_n$ is the surface area of $\S^{n-1}$. Thus we have 
    \begin{align*}
        \left| \Hess \wt{f}^\delta (p) (u,u) - \Hess f (p) (u,u) \right| \le \frac{ L_4 n \delta^2 }{n+2} . 
    \end{align*} 
    We can conclude the proof since the above inequality holds for arbitrary unit vector $u$. \hfill
\end{proof}

\subsection{Zeroth Order Hessian Inversion}

\subsubsection{Hessian Adjugate Estimation by Cramer's Rule} 

Cramer's rule states that the inverse of a nonsingular matrix $A$ equals
\begin{align*} 
    A^{-1} = \frac{1}{\det (A)} \adj (A), 
\end{align*} 
where $\adj (A)$ is the adjugate of matrix $A$. Recall the adjugate of matrix $A$ is $$\adj (A) = \[ (-1)^{i+j} M_{ji} \]_{ \{ 1\le i,j\le n \} } ,$$ where $ M_{ji} = \det \( A_{-ji} \) $ and $A_{-ji}$ is the submatrix of $A$ by removing the $j$-th row and $i$-th column. As suggested by the Cramer's rule, one can estimate inverse of Hessian (up to scaling) by first estimating the unsigned minors of the Hessian and then gather the minors into a matrix. This estimation procedure is summarized in Algorithm \ref{alg:hess-adj-cramer}. 

\begin{algorithm}[h!] 
    \caption{Cramer-Hessian-Adjugate (\texttt{CHA})}   
    \label{alg:hess-adj-cramer} 
    \begin{algorithmic}[1] 
        \STATE \textbf{Input: } number of samples $m$; finite difference step size $\delta$; location for estimation $x$. 
        \STATE Uniformly independently sample $ \{ \(v_{k,ij,ab}, w_{k,ij,ab} \)\}_{1 \le k \le m, 1 \le i,j \le n, 1 \le a,b \le n} $ from $ \S^{n-1}$ (the unit sphere in $\S^{n-1}$). 
        \STATE For all $i,j \in [n] $ and $k \in [m]$, create $n^2$ estimators for the  $(i,j)$-submatrix of $ \[ \nabla^2  \wt{f}^\delta (x) \]_{-ij}$ by 
        \begin{align*} 
            \wh{\text{S}}_{k,ij,ab} 
            = 
            \[ \wh{\H} {f} (x; v_{k,ij,ab}, w_{k,ij,ab}; \delta) \]_{-ij}  , \quad \forall 1 \le i,j,a,b \le n, \; \; \forall k \in [m]. 
        \end{align*} 
        \STATE 
        Create estimators of $ \[ \nabla^2 \wt{f}^\delta (x) \]_{-ij} $ (written $\wh{\text{S}}_{k,ij}$) such that the $(a,b)$-th entry of $ \wh{\text{S}}_{k,ij} $ is the $(a,b)$-th entry of $ \wh{\text{S}}_{k,ij,ab} $ for all $a,b \in [n]$. 
        \STATEx /* In practice, one can use the entry-wise estimators to replace the estimator in Step 3. See Section \ref{sec:single-entry} for more details on entry-wise Hessian estimators. */ 
        \STATE For all $i,j \in [n] $, estimate the unsigned minors by 
        \begin{align*} 
            \wh{M}_{ij} 
            = 
            \frac{1}{m} \sum_{k=1}^m \det \( \wh{\text{S}}_{k,ij} \). 
        \end{align*} 
        /* The determinant can be computed via LU decomposition, QR decomposition, or similar methods. */
        \STATE Estimate the adjugate of Hessian by 
        \begin{align} 
            \ol{\A}_m \wt{f}^\delta (x) = \[ (-1)^{i+j} \wh{M}_{ji} \].  \label{eq:hess-est-cramer}
        \end{align} 
        \STATE \textbf{Output:} $\texttt{CHA} (m, \delta, x) = \ol{\A}_m \wt{f}^\delta (x) .$ 
    \end{algorithmic} 
\end{algorithm}

\begin{proposition} 
    \label{prop:hess-inv-cramer}
    Let $ \CHA (m , \delta, x) $ be the estimator returned by Algorithm \ref{alg:hess-adj-cramer}. If $\nabla^2 \wt{f}^\delta (x)$ is non-singular, it holds that 
    \begin{align*} 
        \E \[ \CHA (m , \delta, x) \] = \det \( \nabla^2 \wt{f}^\delta (x) \) \nabla^{-2} \wt{f}^\delta (x), 
    \end{align*} 
    where $  \nabla^{-2} \wt{f}^\delta (x) : = \[ \nabla^{2} \wt{f}^\delta (x) \]^{-1}$. 
\end{proposition} 

\begin{proof}[Proof. ] 
    
    We will use the notations defined in Algorithm \ref{alg:hess-adj-cramer}. By Lemma \ref{lem:surrogate}, we know that, $ \forall i,j,a,b \in [n], \; \; \forall k \in [m]$, 
    \begin{align*} 
        \E \[ \wh{\text{S}}_{k,ij} \] 
        = \E \[ \wh{\text{S}}_{k,ij,ab} \] 
        = \[ \E \[ \wh{\H} {f} (x; v_{k,ij,ab}, w_{k,ij,ab}; \delta ) \] \]_{-ij} 
        =  
        \[ \nabla^2 \wt{f}^\delta (x) \]_{-ij} . 
    \end{align*} 
    
    Since (i) the determinant of a matrix can be expressed in terms of multiplication and addition of its entries, and (ii) all entries of $\wh{\text{S}}_{k,ij}$ are independent, we have 
    \begin{align*} 
        \E \[ \det \( \wh{\text{S}}_{k,ij} \) \] 
        = 
        \det \( \E \[  \wh{\text{S}}_{k,ij} \] \). 
    \end{align*} 
    
    By a use of the Cramer's rule and the above result, it holds that
    \begin{align*} 
        \E \[ \CHA (m , \delta, x) \] = \E \[ (-1)^{i+j} \wh{M}_{ji} \] = \adj \( \nabla^2 \wt{f}^\delta (x) \) = \det \( \nabla^2 \wt{f}^\delta (x) \) \nabla^{-2} \wt{f}^\delta (x). 
    \end{align*}
\end{proof}

The biggest advantage of the \texttt{CHA} method is that it gives an unbiased estimator of the adjugate matrix of $\nabla^2 \wt{f}^\delta (x)$. Also, Proposition \ref{prop:hess-inv-cramer} hold true even if $\nabla^2 \wt{f}^\delta (x)$ is non-definite. 
However, a shortcoming of the \texttt{CHA} method is its computational expense. For this reason, we introduce the following zeroth order Hessian inversion method, for a special class of Hessian matrices.

\subsubsection{Hessian Inverse Estimation by Neumman Series} 

An approach for computing the inverse of Hessian is via Neumman series. For an invertible matrix $A$ satisfying $ \lim_{p \rightarrow \infty} \( I - A \)^p = 0 $, the Neumann series expands the inverse of $A$ by 
\begin{align*}
    A^{-1} = \sum_{p=0}^\infty \( I - A \)^{-1}. 
\end{align*}
From this observation, we can first estimate the Hessian, and then estimate the inverse by the Neumann series. Previously, \citet{agarwal2017second} studied fast Neumann series based Hessian inversion using first-order information. Here a similar result can be obtained using zeroth-order information only. This zeroth-order extension of \citep{agarwal2017second} is summarized in Algorithm \ref{alg:hess-inv-neumman}. 

\begin{algorithm}[H] 
    \caption{Neumman-Hessian-Inverse (\texttt{NHI})} 
    \label{alg:hess-inv-neumman} 
    \begin{algorithmic}[1] 
        \STATE \textbf{Input: } number of samples $(m_1, m_2, m_3)$; finite difference step size $\delta$; location for estimation $x$. 
        \STATE Uniformly independently sample $ \{ \(v_{ijk}, w_{ijk} \)\}_{1 \le i \le m_1, 1 \le j \le m_2, 1 \le k \le m_3} $ from $ \S^{n-1}$. 
        \STATE For all $i,j,k$, compute 
        \begin{align*} 
            \wh{\H} {f} (x; v_{ijk}, w_{ijk}; \delta), 
        \end{align*} 
        as defined in (\ref{eq:def-euc-stab}). 
        \STATE Compute 
        \begin{align} 
            \ol{\H}_{m_1,m_2,m_3}^{-1} \wt{f}^\delta (x) =  \frac{1}{m_1} \sum_{i=1}^{m_1} \( I + \sum_{h=1}^{m_2} \prod_{j=1}^{h} \( I - \frac{1}{m_3} \sum_{k=1}^{m_3 } \wh{\H} {f} (x; v_{ijk}, w_{ijk}; \delta) \) \) .  \label{eq:def-inv-hess-neumann} 
        \end{align} 
        \STATE \textbf{Output:} $\texttt{NHI} (m_1,m_2,m_3, \delta, x) = \ol{\H}_{m_1,m_2,m_3}^{-1} \wt{f}^\delta (x) .$ 
    \end{algorithmic} 
\end{algorithm}

\begin{proposition}
    \label{prop:hess-est-neumann} 
    Suppose $f$ is twice-differentiable, $\alpha$-strongly convex and $\beta$-smooth with $\beta < 1$. 
    Then it holds that 
    \begin{align} 
        \left\| \E \[ \texttt{NHI} (m_1,m_2,m_3, \delta, x) \] - \nabla^{-2} \wt{f}^\delta (x)   \right\| \le \frac{ \( 1 - \alpha \)^{m_2 + 1} }{ \alpha}, 
    \end{align} 
    where $ \nabla^{-2} \wt{f}^\delta (x) := \[ \nabla^2 \wt{f}^\delta (x) \]^{-1} $. 
\end{proposition} 



\begin{proof}[Proof. ] 
    
    Since $ f $ is $\alpha$-strongly convex, it holds that, for any $x ,y, v,w \in \R^n$,
    \begin{align*} 
        f ( x + v + w ) \ge f ( y + v + w ) + \( x - y \)^\top \nabla f (y + v + w)  + \frac{\alpha}{2} \left\| x - y \right\|^2 . 
    \end{align*} 
    Integrating both $v$ and $w$ over $\delta \B^n$ gives that 
    \begin{align*} 
        \wt{f}^\delta (x) \ge \wt{f}^\delta (y) + (x-y)^\top \nabla \wt{f}^\delta (y) + \frac{\alpha}{2} \left\| x - y \right\|^2, 
    \end{align*} 
    where we use the dominated convergence theorem to interchange the integral and the gradient operator. 
    This shows that $ \wt{f}^\delta $ is also $\alpha$-strongly.

    Since a differentiable function $f$ is $ \beta $-smooth if and only if $   f ( x ) \le f ( y ) + \nabla f (y) ^\top \( x - y \) + \frac{\beta}{2} \left\| x - y \right\|^2 $ for all $x, y \in \R^n$, one can show that $\wt{f}^\delta$ is $\beta$-smooth by repeating the above argument. 
    
    For $ \texttt{NHI} (m_1,m_2,m_3, \delta, x) $, we have 
    \begin{align*} 
        \E \[ \texttt{NHI} (m_1,m_2,m_3, \delta, x) \]
        =& \;  
        I + \sum_{h=1}^{m_2} \prod_{j=1}^{h} \( I - \E \[ \wh{\H} {f} (x; v_{ijk}, w_{ijk}; \delta) \] \) \\ 
        =& \; 
        \sum_{j=0}^{m_2} \( I - \nabla^2 \wt{f}^\delta (x) \)^j  . 
    \end{align*} 
    Since $ \wt{f}^\delta $ is $\alpha$-strongly convex, $\beta$-smooth ($\beta < 1$), and apparently twice-differentiable, we have 
    \begin{align*}
        0 \preccurlyeq I - \nabla^2 \wt{f}^\delta (x) \preccurlyeq \( 1 - \alpha \) I . 
    \end{align*}
    Thus we can bound the bias by 
    \begin{align*} 
        \left\| \E \[ \texttt{NHI} (m_1,m_2,m_3, \delta, x) \] - \nabla^{-2} \wt{f}^\delta (x) \right\|
        \le 
        \sum_{j = m_2 + 1}^{\infty } (1 - \alpha)^j = \frac{ \( 1 - \alpha \)^{m_2 + 1} }{ \alpha} . 
    \end{align*}
\end{proof}



\section{Existing Methods and Experiments}
\label{sec:exp}

\subsection{Existing Methods for Hessian Estimation}
\label{sec:existing}

\subsubsection{Hessian Estimation via Collecting Single Entry Estimations}  
\label{sec:single-entry}
In the Euclidean case, one can fix a canonical coordinate system $\{ \e_i \}_{i \in [n]}$, and the $(i,j)$-th entry of the Hessian matrix of $f $ at $x$ can be estimated by 
\begin{align}  
    \wh{\H}_{ij}^{\entry} {f} (x;\delta) 
    {=}& \;  
    \frac{ 1 }{ 4 \delta^2 } \big( f ( x + \delta \e_i + \delta \e_j ) - f ( x + \delta \e_i - \delta \e_j )  \nonumber \\
    &- f ( x - \delta \e_i + \delta \e_j )  + f ( x - \delta \e_i - \delta \e_j ) \big) .   \label{eq:hess-entry-entry} 
\end{align}
One can then gather the entries to obtain a Hessian estimator: 
\begin{align} 
    \wh{\H}^{\entry} {f} (x;\delta) = \[ \wh{\H}_{ij}^{\entry} {f} (x;\delta) \]_{i,j \in [n]} . \label{eq:hess-entry}
\end{align}

This estimator could perhaps date back to classic times when the finite difference principles were first used. Yet it needs at least $ \Omega (n^2) $ zeroth order samples to produce an estimator in an $n$-dimensional space. 
Previously, \citet{balasubramanian2021zeroth} designed a Hessian estimator based on the Stein's identity \citep{10.1214/aos/1176345632}. Their estimator only needs $O(1)$ zeroth-order function evaluations. This method is discussed in the next section. 

\subsubsection{Hessian Estimation via the Stein's identity}   

A classic result for Hessian computation is the Stein's identity (named after Charles Stein), as stated below. 
    
\begin{theorem}[Stein's identity]
\label{thm:rad-hess} 
    Consider a smooth function $f : \R^n \rightarrow \R$. For any point $x \in \R^n$, it holds that 
    \begin{align*} 
        \nabla^2 f (x ) = \frac{1}{ 2 } \E \[ \( u u^\top - I \) D_{uu} f (x) \], 
    \end{align*} 
    where 1. $u {\sim} \mathcal{N} (0, I )$, and 2. 
    \begin{align*} 
        D_{uu} f (x) =  \lim_{\tau \rightarrow 0}  \frac{ f ( x +  \tau u  ) - 2 f ( x ) + f ( x - \tau u  ) }{\tau^2} . 
    \end{align*} 
\end{theorem} 

\begin{proof}[Proof. ]
    For completeness, a convenient proof of Theorem  \ref{thm:rad-hess} is provided in the Appendix. \hfill 
\end{proof}

One can estimate Hessian using the Stein's identity \citep{balasubramanian2021zeroth}: 
\begin{align}
    \wh{\H}^{\Stein} {f} (x; u; \delta) = \frac{ {f} (x + \delta u ) - 2 {f} (x ) + {f} (x - \delta u ) }{ 2 \delta^2 } (u u^\top - I), \label{eq:stein-hess} 
\end{align} 
where $ u \sim \mathcal{N} (0,I) $ is a standard Gaussian vector. A bias bound for (\ref{eq:stein-hess}) is in Theorem \ref{thm:stein-hess}. 

\begin{theorem}[\citet{li2020stochastic,balasubramanian2021zeroth}] 
    \label{thm:stein-hess} 
    Let $f$ have $L_2$-Lipschitz Hessian: There exists a constant $L_2$ such that $ \| \nabla^2 f (x) - \nabla^2 f (x') \| \le L_2 \| x - x' \|$ for all $x,x' \in \R^n$. The estimator in (\ref{eq:stein-hess}) satisfies 
    \begin{align*} 
        \left\| \E \[ \wh{\H}^{\Stein} {f} (x; u; \delta) \] - \nabla^2 f (x) \right\| \le \frac{ L_2 \( n + 6 \)^{\frac{5}{2}} \delta }{4} , 
    \end{align*} 
    for any $x \in \R^n$ and any function $f : \R^n \rightarrow \R $ with $L_2$-Lipschitz Hessian. 
\end{theorem} 

The estimator (\ref{eq:stein-hess}) improves the entry-wise estimator in the sense that  one only needs $O(1)$ samples to produce an estimator. However, its error bound given by Theorem \ref{thm:stein-hess} is worse than that of (\ref{eq:def-euc-stab}) in Theorem \ref{thm:hess-bias}. A more detailed discussion on the error bounds is in Remark \ref{remark}. 

\begin{remark} 
    \label{remark} 
    We need to note that our estimator (\ref{eq:def-euc-stab}) and the estimator via Stein's method (\ref{eq:stein-hess}) have different finite-difference step size. More specifically, $ \E_{v,w \overset{i.i.d.}{\sim} \S^{n-1} } \[ \delta \| v + w \| \] = \Theta \( \delta \) $ for (\ref{eq:def-euc-stab}) and $\E_{u \sim \mathcal{N} (0,I_n)} \[ \delta \| u \| \] = \Theta \( \sqrt{n} \delta \) $ for (\ref{eq:stein-hess}). To compare the bias bounds for (\ref{eq:def-euc-stab}) and (\ref{eq:stein-hess}) using the same (expected) finite-difference step size, we need to downscale the bound in Theorem \ref{thm:stein-hess} by a factor of $\sqrt{n}$. After this downscaling, the error bound for the Stein-type estimator (\ref{eq:stein-hess}) is $ O \( L_2 n^2 \delta \) $ which is worse than the bias bound bound for our estimator (\ref{eq:def-euc-stab}). In the experiments, we down scale the finite-difference step size when studying all results of Stein's method estimator. 
\end{remark} 

As discussed in Remark \ref{remark}, a difference between (\ref{eq:def-euc-stab}) and (\ref{eq:stein-hess}) is that they sample random vectors from different distributions: uniformly random vector from the unit sphere for (\ref{eq:def-euc-stab}) and standard Gaussian vector for (\ref{eq:stein-hess}). 
High moments of uniformly random vectors from the unit sphere are smaller than Gaussian vectors of same expected norm. More specifically, The $k$-th moment of (norm of) a standard Gaussian vector $ v \sim \mathcal{N} (0, I_n) $ that is downscale by a factor of $\sqrt{n}$ is 
\begin{align*}
    &\; n^{-k/2} \E \[ \| v \|^{k} \] \\
    =&\; 
    \frac{n^{-k/2}}{ \( 2\pi \)^{-n/2} } \int_0^\pi \int_0^\pi \cdots \int_0^{2\pi} \int_{0}^\infty r^{k + n - 1}  e^{ -\frac{ r^2 }{2}  } \, dr \, \sin ^{n-2}(\varphi _{1})\sin ^{n-3}(\varphi _{2})\cdots \sin(\varphi _{n-2})\,d\varphi _{1}\,d\varphi _{2}\cdots d\varphi _{n-1} \, \\ 
    =&\; 
    n^{-k/2} \( 2\pi \)^{-n/2} A_n \int_{0}^\infty r^{k + n - 1}  e^{ -\frac{ r^2 }{2}  } \, dr \,  \tag{$A_n$ is the surface area of the Euclidean unit sphere $\S^{n-1}$} \\ 
    =& \; 
    n^{-k/2} \( 2\pi \)^{-n/2} \frac{ 2 \pi^{n/2} }{\Gamma \( \frac{n}{2} \)} 2^{\frac{n+k}{2} - 1 } \Gamma \( \frac{k + n}{2} \) \\ 
    =& \; 
    n^{-k/2} 2^{\frac{k}{2} } \frac{ \Gamma \( \frac{k + n}{2} \) }{\Gamma \( \frac{n}{2} \)} \\
    \sim&\; 
    n^{-k/2} 2^{\frac{k}{2} } \frac{ {\sqrt {\frac {2\pi }{ \frac{k+n}{2} }}}\,{\left({\frac { \frac{k+n}{2} }{e}}\right)}^{ \frac{k+n}{2} } }{ {\sqrt {\frac { 2\pi }{ \frac{n}{2} }}}\,{\left({\frac{ \frac{n}{2} }{e}}\right)}^{ \frac{n}{2} } } \tag{by Stirling's approximation} \\
    =& \; 
    \( \frac{e n}{2} \)^{-k/2} \sqrt{  \frac{n}{n+k} } \( \frac{n+k}{2} \)^{n/2+k/2} \( \frac{n}{2} \)^{-n/2}
\end{align*} 
which clearly grows very fast with $k$ for large $k$ and for any fixed $n$. On contrary, the moments of (norm) of the vector uniformly sampled from the unit sphere are all $ 1 $. This difference implies that our estimator tends to have smaller higher order moments. 

\subsection{Numerical Results} 

We test the performance of our estimator against the previous two estimators in noisy environments. Before proceeding, we re-define some notations for the estimators, so that the estimators are tested on the same ground and noise is properly taken into consideration. The estimators we will empirically study are
\begin{enumerate} 
    \item Our new estimator: 
    \begin{align} 
        & \; \wh{\H}^{\text{new}} f (p; m ; \delta) \nonumber \\
        =& \;  
        \sum_{ k=1 }^{ \lfloor \frac{m}{4} \rfloor } \frac{n^2}{ \delta^2} \big[ \epsilon_k +  f ( \Exp_p ( \delta v_k + \delta w_k) ) - f ( \Exp_p ( -\delta v_k + \delta w_k ) ) \nonumber  \\
        & - f ( \Exp_p ( \delta v_k - \delta w_k) ) + f ( \Exp_p ( - \delta v_k - \delta w_k) )  \big] ( v_k \otimes w_k + w_k \otimes v_k ) ,  \label{eq:redef-new}
     \end{align} 
    where $v_k, w_k \overset{i.i.d.}{\sim} \S_p $, and $\epsilon_k$ is a mean-zero noise that is independent of all other randomness. 
    \item The Stein's estimator: 
    \begin{align} 
        &\; \wh{\H}^{\text{Stein}} f (p; m ; \delta) \nonumber \\
        =& \;  
        \sum_{ k=1 }^{ \lfloor \frac{m}{3} \rfloor } \frac{ n }{ 2 \delta^2 } \[ {f} \( \Exp_p \( \frac{ \delta u_k}{ \sqrt{n} } \) \) - 2 {f} ( p ) + {f} \( \Exp_p \( \frac{ -\delta u_k}{ \sqrt{n} } \) \) + \epsilon_k \] \nonumber \\
        & \qquad \cdot (u_k \otimes u_k - I) , \label{eq:redef-stein}
    \end{align} 
    where $ u_k \overset{i.i.d.}{\sim} \mathcal{N} (0, I) $ (the standard Gaussian in $ T_p \M $), $I$ is the identity map from $ T_p \M $ to itself (As a bilinear form, $I (u,v) = \< u,v \>_p$ for any $u,v \in T_p \M$.), and $\epsilon_k$ is a mean-zero noise that is independent of all other randomness. 
    \item The entry-wise estimator: 
    \begin{align} 
        \wh{\H}^{\text{entry}} f (p; m ; \delta) = \[ \wh{\H}_{ij}^{\text{entry}} f (p; m ; \delta ) \]_{i,j \in [n]} , \label{eq:redef-entry}
    \end{align}
    where 
    \begin{align}
        \wh{\H}_{ij}^{\text{entry}} f (p; m ; \delta ) 
        =& \;  
        \frac{ 1 }{ 4 \delta^2 } \sum_{ k = 1 }^{ \lfloor \frac{m}{4 n^2} \rfloor } \big( f ( \Exp_p ( \delta \e_i + \delta \e_j ) ) - f ( \Exp_p ( \delta \e_i - \delta \e_j ) )  \nonumber \\
        &- f ( \Exp_p ( - \delta \e_i + \delta \e_j ) )  + f ( \Exp_p ( - \delta \e_i - \delta \e_j ) ) + \epsilon_k \big) , \nonumber 
    \end{align} 
    $\{\e_i\}_i$ is the local normal coordinate for $T_p \M$, and $\epsilon_k$ is a mean-zero noise that is independent of all other randomness. 
\end{enumerate} 

\begin{table}[b!]
    \caption{Manifolds used for testing. The local metric near $p$ is implicitly specified by the exponential map. \label{tab}  } 
    
    \begin{tabular}{c|c|c|c} 
      Manifold  & $ p$ ($p \in \R^{n+1}$) & $ h ( x ) $, $x \in T_p \M \cong \R^n$ & $\Exp_p (v)$ \\ \hline \hline 
     (I) & $p=0$ & $h (x) = 0$ & $ (v,h (v)) $ \\ 
     (II) & $p=0$ &  $h (x) = 1 - \sqrt{  1 - \sum_{i=1}^n x_i^2 } $ & $ (v, h (v) ) $ \\ 
     (III) & $p=0$ &  $h (x) = \sum_{i=1}^{n/2} x_i^2 - \sum_{i = n/2+1}^n x_i^2 $ & $ (v, h (v) ) $ 
    \end{tabular} 
\end{table} 

\begin{figure}[h!]
    \centering
    \subfloat[][$\delta = 0.05$]{\includegraphics[width = 0.45\textwidth]{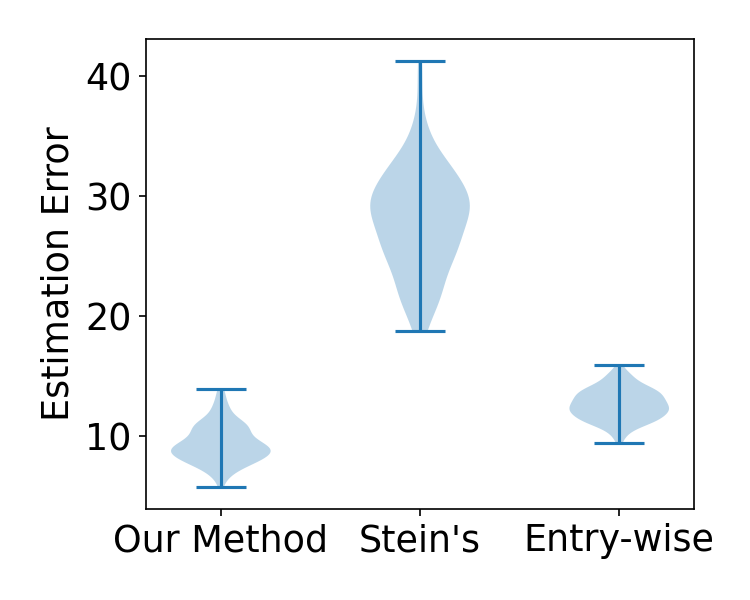} } 
    \subfloat[][$\delta = 0.1$]{\includegraphics[width = 0.45\textwidth]{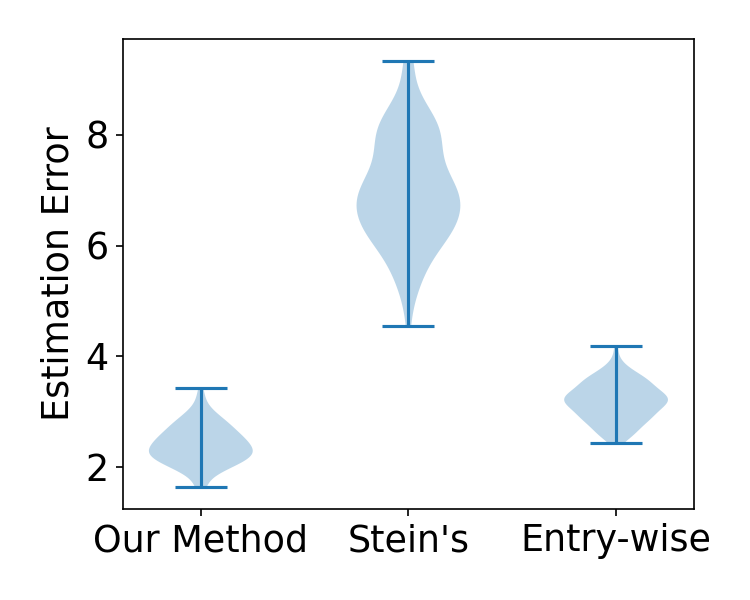} } \\
    \begin{centering} 
        \subfloat[][$\delta = 0.2$]{\includegraphics[width = 0.45\textwidth]{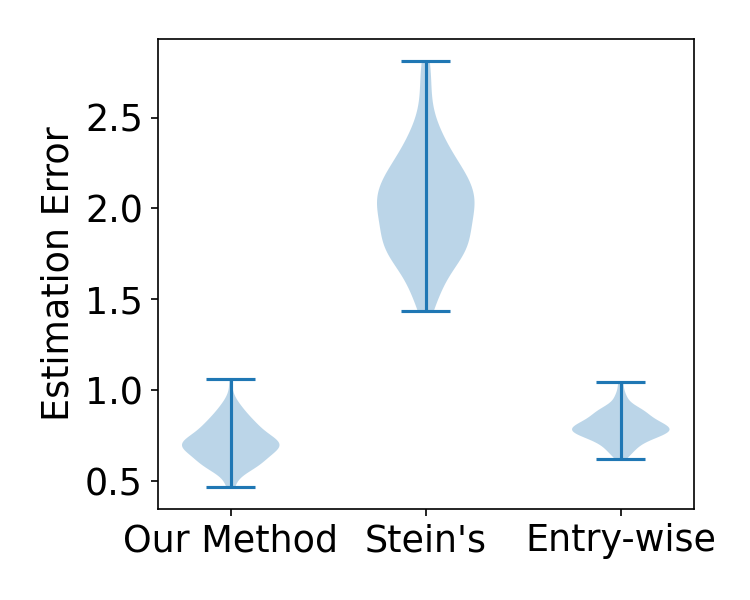} } 
    \end{centering} 
    \caption{Results for the Manifold (I). Each violin plot summarizes estimation error of 100 estimations. More specifically, each estimation in this figure uses $m = 3840$ function evaluations, and 100 estimations are used to generate one violin plot. On the $x$-axis, ``Our method'' corresponds to our estimator (\ref{eq:redef-new}); ``Stein's'' corresponds to the Stein's method (\ref{eq:redef-stein}); ``Entry-wise'' corresponds to the entry-wise estimator (\ref{eq:redef-entry}). Subfigures (a), (b), (c) corresponds to $\delta = 0.05$, $\delta = 0.1$, $\delta = 0.2$. \label{fig}  } 
\end{figure}

\begin{table}[b!]
    \caption{Timing results in seconds, rounded to 1e-4 accuracy. In the table, ``Our method'' stands for the estimator (\ref{eq:redef-new}), and ``Stein's'' stands for the estimator (\ref{eq:redef-stein}). The time consumption of the estimators are divided into three parts: (1) sampling time, used for generating the random vectors (uniformly random unit vectors for our methods, and standard Gaussian vectors for the Stein's method), (2) evaluation time, used for accessing function values, and (3) computation time, used for matrix manipulations (e.g., outer product computation). In the timing experiments, both estimators (\ref{eq:redef-new}) and (\ref{eq:redef-stein}) use $m = 10,000$, $\delta = 0.05$ and $n = 8$. All timing results are averaged 10 times, presented in a ``mean $\pm$ standard deviation'' format. The last two columns are cumulative, to avoid fast memory access to saved data. \label{tab:timing} } 
    \begin{tabular}{c|c|c|c} 
    \hline 
        & Sampling & Sampling + Evaluation & \makecell{Sampling + Evaluation \\ + Computation} \\ \hline \hline 
    Our method & $0.1580 \pm 0.0031$ & $0.5763 \pm 0.0036$ & $0.6977\pm 0.0040$ \\ 
    Stein's  & $0.0257 \pm 0.0012$ & $0.3020\pm 0.0024$ & $0.4222\pm 0.0028$ \\ 
     \hline 
    \end{tabular} 
\end{table} 

Strictly speaking, the noises $\epsilon_k$ corrupt all the zeroth-order function value observations. Specifically, the notation $ \epsilon_k +  f ( \Exp_p ( \delta v_k + \delta w_k) ) - f ( \Exp_p ( -\delta v_k + \delta w_k ) )  - f ( \Exp_p ( \delta v_k - \delta w_k) ) + f ( \Exp_p ( - \delta v_k - \delta w_k) )  $ should be understood in the following way. All four functions values $ f ( \Exp_p ( \delta v_k + \delta w_k) )$, $ f ( \Exp_p ( -\delta v_k + \delta w_k ) )$, $f ( \Exp_p ( \delta v_k - \delta w_k) ) $, and $ f ( \Exp_p ( - \delta v_k - \delta w_k) ) $ are corrupted with mean-zero and independent noise and not directly observable. Note that all previously discussed bias bounds hold when the function evaluations are corrupted by independent mean-zero noise. 

The above notations allow us to put all the estimators on the same ground more easily. With the new notations, all $ \wh{\H}^{\text{new}} f (p; m ; \delta )  $, $ \wh{\H}^{\text{Stein}} f (p; m ; \delta )  $ and $ \wh{\H}^{\text{entry}} f (p; m ; \delta )  $ uses $ m $ functions value observations and have an expected finite difference step size $\Theta (\delta)$. The redefining of the estimators is needed since 1. the entry-wise estimator needs more samples to output an estimate, and 2. the default Stein's method in expectation uses a larger finite-difference step-size, as discussed in Remark \ref{remark}. 

\begin{remark}
The estimator we introduced (\ref{eq:redef-new}) has a practical advantage over that via the Stein's identity (\ref{eq:redef-stein}). The reason is that estimators based on the Stein's identity requires one to explicitly know the identity map from $T_p \M$ to itself. This map may or may not admit an easy numerical representation. For example, for the real Stiefel's manifold $\St (n,k) = \{ X \in \R^{n \times k}: X^\top X = I \}$, we know that the map 
    \begin{align*}
        P_X Z := (I - XX^\top  ) Z + \frac{1}{2} X \( X^\top Z - Z^\top X \) 
        , \qquad \forall Z \in \R^{n \times k}, 
    \end{align*} 
    is the identity from $T_X \St (n,k)$ to itself \citep[e.g.,][]{absil2009optimization}. Also, this map projects any $ Z \in \R^{n \times k} $ onto $T_X \St (n,k)$. 
    Extracting a numerical representation of this map $P_X$ may not be easy. On contrary, for any $Z_1, Z_2 \in T_X \St (n,k)$, computing $ Z_1 \otimes Z_2 $ is tractable. More specifically, one can use the following procedure to obtain a uniformly random vector from the unit sphere in $T_X \St (n,k)$ for a given $X \in \St (n,k)$. 
    One can (1) sample an $i.i.d.$ Gaussian matrix $G$ from $\R^{n,k}$, (2) compute $P_X G$, and (3) normalize $ P_X G $ with respect to the Frobenius inner product. By rotational invariance (of the standard Gaussian distribution and the Frobenius norm), this procedure outputs a uniformly random unit matrix in $T_X \St (n,k)$. Once we have the unit vectors in $T_X \St (n,k)$, we can numerically compute their tensor product. 
\end{remark}

All three methods are tested using the following test function, defined using the standard Cartesian coordinate system in $\R^{n+1}$,
\begin{align*} 
    f (x) = \sum_{i=1}^{n+1} \cos ( x_i ) + \exp ( x_1 x_2 ). 
\end{align*} 
Every function evaluation is corrupted with an independent noise sampled from $\mathcal{N} (0, 0.0025)$. The estimators are tested over three manifolds in $\R^{n+1}$. More details about the three manifolds are in Table \ref{tab}. In all settings, we set the number of total function evaluation $ m = 3840 $ and dimension of manifold $n = 8$. The results for manifold (I), the Euclidean case, is in Figure \ref{fig}. Results for manifold (II) and manifold (III) are in Appendix \ref{app:exp}. In Figure \ref{fig} (and Figures \ref{fig2} and \ref{fig3} in Appendix \ref{app:exp}), the errors on the $y$-axis plots the norm of the difference between the empirical estimator and the truth: 
\begin{align*} 
    \left\| \wh{\H} f (p; v,w; \delta) - \Hess {f} (p) \right\| . 
\end{align*}

    

\subsection{Time Efficiency Comparison}

We compare the time efficiency of our method and the estimator based on the Stein's identity. 
In general, one expects estimators based on the Stein's identity to be more time-efficient. Main reasons for this include that the estimator based on the Stein's identity needs only 3 function evaluations instead of 4. In practice, the function evaluations may or may not be cheap. When the functions evaluations are expensive, we may expect that estimators based on the Stein's identity approximately saves 1/4 time, compared to our method. When the function evaluations are cheap, our estimator (\ref{eq:redef-new}) is in general more time consuming as well, since more sampling and more matrix computations need to be carried out. 

In Table \ref{tab:timing}, we compare the running time of (\ref{eq:redef-new}) and (\ref{eq:redef-stein}). All timing experiments use the same benchmark function and underlying manifold as Figure \ref{fig}. We use $n = 8$, $m = 10,000$ and $\delta = 0.05$ for timing experiments. All timing experiments are carried out in an environment with 
\begin{itemize}
    \item 10 cores and 20 logical processors with a maximum speed of 2.80 GHz; 
    \item 32GB RAM;
    \item Windows 11 22000.832; 
    \item Python 3.8.8. 
\end{itemize} 


\section{Conclusion}

In this paper, we study Hessian estimators over Riemannian manifolds. 
We design a new estimator, such that for real-valued analytic functions over an $n$-dimensional complete analytic Riemannian manifold, our estimator achieves an $ O ( \gamma \delta^2 ) $ expected error, where $\gamma $ depends both on the Levi-Civita connection and the function $f$, and $\delta$ is the finite difference step size. Downstream computations of Hessian inversion is also studied. Empirical studies show supremacy of our method over existing methods. 

\section*{Data Availability Statement}

No new data were generated or analysed in support of this paper. Code for the experiments is available at \url{https://github.com/wangt1anyu/zeroth-order-Riemann-Hess-code}. 


\appendix 

\section{Proof of Theorem \ref{thm:rad-hess}}


\begin{proof}[Proof of Theorem \ref{thm:rad-hess}. ] 
    
    Consider $ \E \[ u_k u_h  u_{i} u_{j} \partial_i \partial_j \] $ for any $k,h,i,j \in [n]$. 
    
    When $(k,h) = (i,j)$, one has $ \E \[ u_k u_h u_i u_j \partial_i \partial_j \] = \E \[ u_i^2 u_j^2 \partial_i \partial_j \] $. In this case, it holds that 
    \begin{align*} 
        \E \[ u_i^2 u_j^2 \partial_i \partial_j \] 
        &= \partial_k \partial_h \quad \text{for } i \neq j \quad \text{and} \quad \E \[ u_i^4 \partial_i \partial_i \] = 3 \partial_k \partial_k , \quad \text{for } i = j. 
    \end{align*} 
    
    When $(k,h) \neq (i,j)$, $ i = j$ and $k = h$, we have $ \E \[ u_k^2 u_i^2 \partial_i \partial_j \] = \partial_i \partial_i $. 
    
    When $(k,h) \neq (i,j)$, $ i = j$ and $k\neq h$, we have $ \E \[ u_k u_h u_i u_j \] = 0 $. 
    
    When $(k,h) \neq (i,j)$, $ i \neq j$ and $k= h$, we have $ \E \[ u_k u_h u_i u_j \] = 0 $.
    
    When $(k,h) \neq (i,j)$, $ i \neq j$, $k\neq h$, $k = j$ and $h=i $, we have $ \E \[ u_k u_h u_i u_j \partial_i \partial_j \] = \E \[ u_i^2 u_j^2 \partial_h \partial_k \] = \partial_h \partial_k $. 
    
    When $(k,h) \neq (i,j)$, $ i \neq j$, $k\neq h$, $k = i$ and $h=j $, we have $ \E \[ u_k u_h u_i u_j \partial_i \partial_j \] = \E \[ u_i^2 u_j^2 \partial_k \partial_h \] = \partial_k \partial_h $. 
    
    When $(k,h) \neq (i,j)$, $ i \neq j$, $k\neq h$ and $k \neq j$, we have $ \E \[ u_k u_h u_i u_j \] = 0 $. 
    
    When $(k,h) \neq (i,j)$, $ i \neq j$, $k\neq h$ and $h \neq i$, we have $ \E \[ u_k u_h u_i u_j \] = 0 $. 
    
    Now using Einstein's notation and combining all above cases give
    \begin{align} 
        \E \[  u^k u_h  u^i u_j \partial_i \partial^j \] 
        =& \;
        \partial^k \partial_h (1 - \delta_{k}^h ) +  \partial^h \partial_k (1 - \delta_{h}^k ) + \delta_{k}^h  \partial_i \partial^i + 2  \partial^k \partial_h \delta_k^h \nonumber \\
        \overset{(i)}{=}& \;
        2  \partial^k \partial_h + \delta_{k}^h  \partial_i \partial^i , \nonumber
    \end{align} 
    where $\delta_k^h$ is the Kronecker's delta. 
    
    Since $D_{uu} f(p) = u^{i} u_{j} \partial_i \partial^j f (x)$ for all $u$ and $x$, we can write $ u u^\top D_{uu} f (x) = u^k u_h  u^i u_j \partial_i \partial^j f (x) $. Thus rearranging terms in $(i)$ gives 
    \begin{align*}
        \E \[ u u^\top D_{uu} f (x) \] \overset{(ii)}{=} 2  \nabla^2 f (x) + \( \Delta f (x) \) I, 
    \end{align*} 
    where $ \Delta =  \partial_i \partial^i $ is the Laplace operator. 
    
    Since $ \E \[ D_{uu} f (x) \] = \E \[ u^i u_j \partial_i \partial^j f (x) \] = \delta_i^j \partial_i \partial^j f (x) = (\Delta f (x)) I $, rearranging terms in $(ii)$ concludes the proof. \hfill
\end{proof}

\section{Additional Experimental Results}
\label{app:exp}

\begin{figure}[h!]
    \centering
    \subfloat[][$\delta = 0.05$]{\includegraphics[width = 0.3 \textwidth]{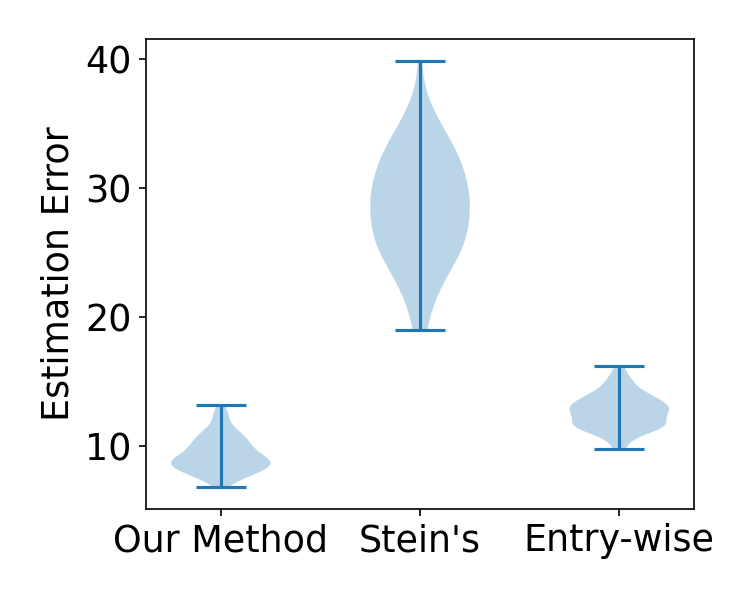} } 
    \subfloat[][$\delta = 0.1$]{\includegraphics[width = 0.3 \textwidth]{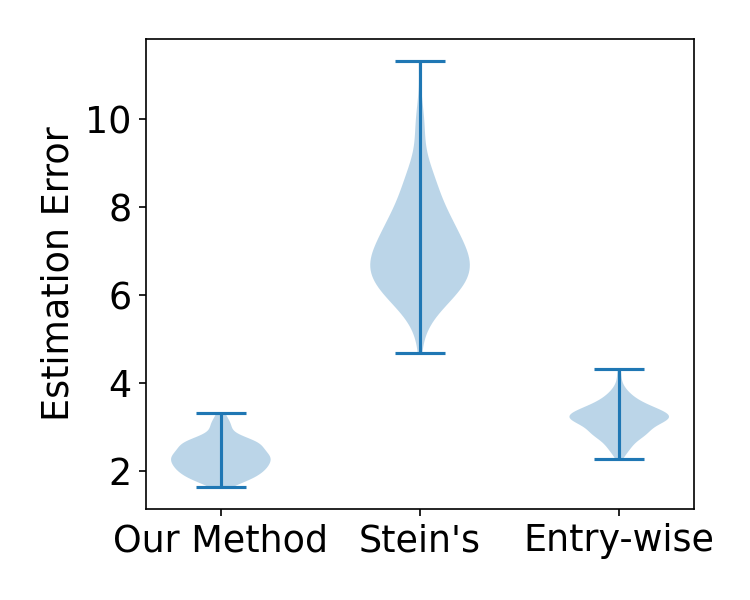} } 
        \subfloat[][$\delta = 0.2$]{\includegraphics[width = 0.3 \textwidth]{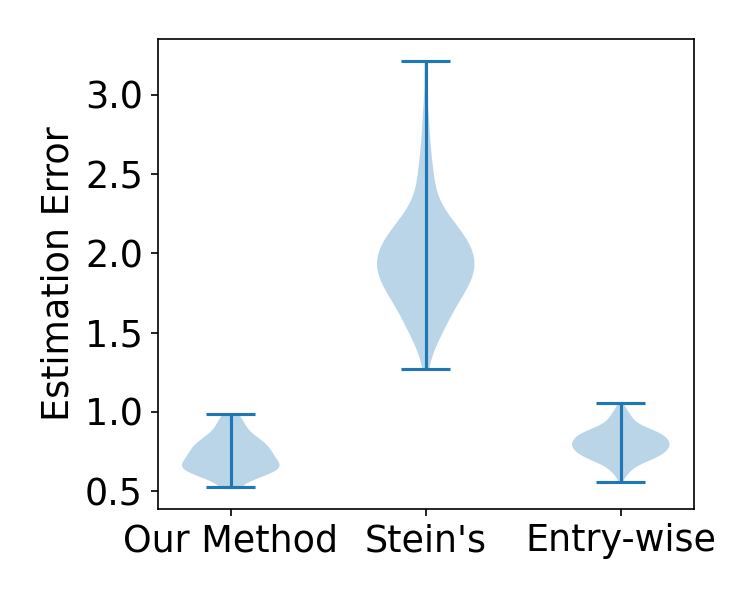} } 
    \caption{Results for the Manifold (II). Each violin plot summarizes estimation error of 100 estimations. Specifically, each estimation in this figure uses $m = 3840$ function evaluations, and 100 estimations are used to generate one violin plot. On the $x$-axis, ``Our method'' corresponds to our estimator (\ref{eq:redef-new}); ``Stein's'' corresponds to the Stein's method (\ref{eq:redef-stein}); ``Entry-wise'' corresponds to the entry-wise estimator (\ref{eq:redef-entry}). Subfigures (a), (b), (c) corresponds to $\delta = 0.05$, $\delta = 0.1$, $\delta = 0.2$. } 
    \label{fig2} 
\end{figure}

\begin{figure}[h!]
    \centering
    \subfloat[][$\delta = 0.05$]{\includegraphics[width = 0.3 \textwidth]{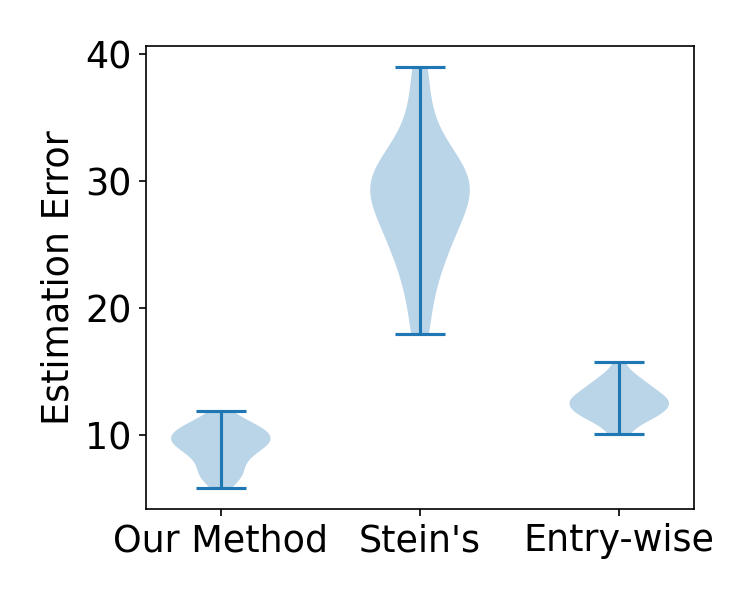} } 
    \subfloat[][$\delta = 0.1$]{\includegraphics[width = 0.3 \textwidth]{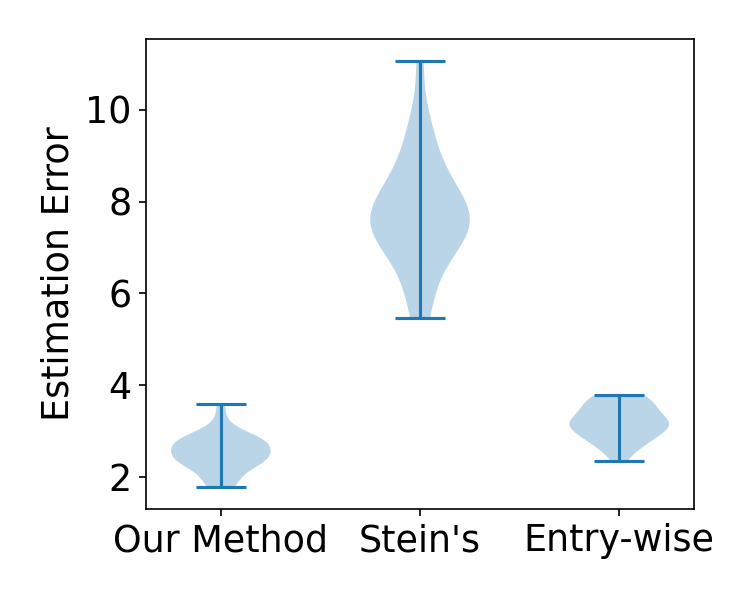} } 
        \subfloat[][$\delta = 0.2$]{\includegraphics[width = 0.3 \textwidth]{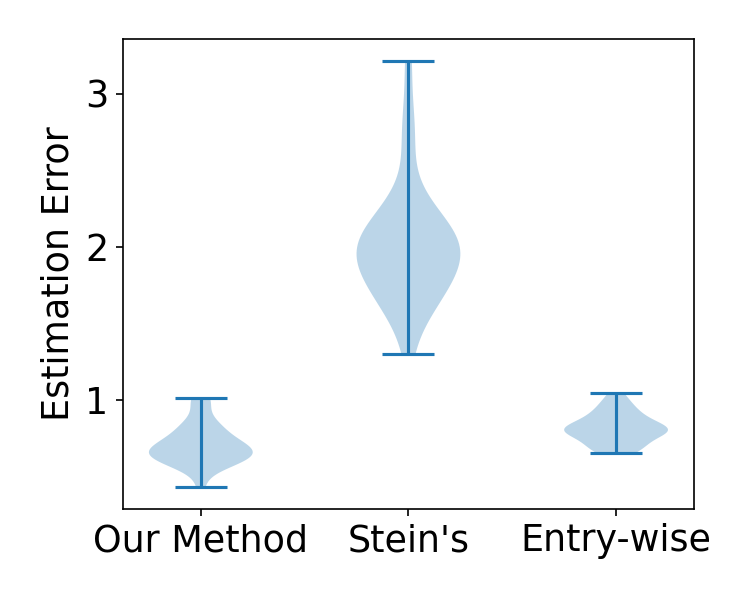} } 
    \caption{Results for the Manifold (III). 
    Each violin plot summarizes estimation error of 100 estimations, with the estimators defined in. Specifically, each estimation in this figure uses $m = 3840$ function evaluations, and 100 estimations are used to generate one violin plot. On the $x$-axis, ``Our method'' corresponds to our estimator (\ref{eq:redef-new}); ``Stein's'' corresponds to the Stein's method (\ref{eq:redef-stein}); ``Entry-wise'' corresponds to the entry-wise estimator (\ref{eq:redef-entry}). Subfigures (a), (b), (c) corresponds to $\delta = 0.05$, $\delta = 0.1$, $\delta = 0.2$. 
    } 
    \label{fig3} 
\end{figure}

\end{document}